\documentclass{article}

\usepackage[preprint]{neurips_2026}
\makeatletter
\renewcommand{\@noticestring}{%
  Code: \url{https://github.com/armaan-abraham/lql}
}
\makeatother

\usepackage[utf8]{inputenc}
\usepackage[T1]{fontenc}

\DeclareUnicodeCharacter{00B7}{\ensuremath{\cdot}}
\DeclareUnicodeCharacter{2010}{-}
\DeclareUnicodeCharacter{0263}{\ensuremath{\gamma}}
\DeclareUnicodeCharacter{03A6}{\ensuremath{\Phi}}
\DeclareUnicodeCharacter{03B1}{\ensuremath{\alpha}}
\DeclareUnicodeCharacter{03B3}{\ensuremath{\gamma}}
\DeclareUnicodeCharacter{03BB}{\ensuremath{\lambda}}
\DeclareUnicodeCharacter{03C0}{\ensuremath{\pi}}
\DeclareUnicodeCharacter{03C3}{\ensuremath{\sigma}}
\DeclareUnicodeCharacter{03C7}{\ensuremath{\chi}}
\usepackage{microtype}
\usepackage{graphicx}
\usepackage{subcaption}
\usepackage{booktabs}
\usepackage{multirow}
\usepackage{makecell}
\usepackage{placeins}
\usepackage{float}
\usepackage{wrapfig}
\usepackage{enumitem}
\usepackage{xcolor}
\usepackage{nicefrac}

\usepackage{amsmath}
\usepackage{amssymb}
\usepackage{amsfonts}
\usepackage{mathtools}
\usepackage{amsthm}

\newcommand{\barnum}[1]{\ensuremath{\overline{\text{#1}}}}

\usepackage{algorithm}
\usepackage{algorithmic}

\usepackage{hyperref}
\usepackage{url}
\usepackage[capitalize,noabbrev]{cleveref}

\theoremstyle{plain}
\newtheorem{theorem}{Theorem}[section]

\newtheorem{lemma}[theorem]{Lemma}

\theoremstyle{definition}

\theoremstyle{remark}

\usepackage[disable,textsize=tiny]{todonotes}

\title{Long-Horizon Q-Learning: Accurate Value Learning via n-Step Inequalities}

\author{%
  Armaan A. Abraham \qquad Lucy Xiaoyang Shi \qquad Chelsea Finn \\
  Stanford University \\
  \texttt{armaana@stanford.edu} \\
}

\begin{document}

\maketitle

\begin{abstract}

Off-policy, value-based reinforcement learning methods such as Q-learning are appealing because they can learn from arbitrary experience, including data collected by older policies or other agents. In practice, however, bootstrapping makes long-horizon learning brittle: estimation errors at later states propagate backward through temporal-difference (TD) updates and can compound over time. We propose \emph{long-horizon Q-learning} (LQL), which introduces a principled backstop against compounding error when learning the optimal action-value function. LQL builds on a prior \emph{optimality tightening} observation: any realized action sequence lower-bounds what the optimal policy can achieve in expectation, so acting optimally earlier should not be worse than following the observed actions for several steps before switching to optimal behavior. Our contribution is to turn this inequality into a practical stabilization mechanism for Q-learning by using a hinge loss to penalize violations of these bounds. Importantly, LQL computes these penalties using network outputs already produced for the TD error, requiring no auxiliary networks and no additional forward passes relative to Q-learning. When combined with multiple state-of-the-art methods on a range of online and offline-to-online benchmarks, LQL consistently outperforms both 1-step TD and n-step TD learning at similar runtime.

\end{abstract}

\section{Introduction}

\begin{wrapfigure}{r}{0.55\textwidth}
    \vspace{-2.5\baselineskip}
    \centering
    \includegraphics[width=0.55\textwidth]{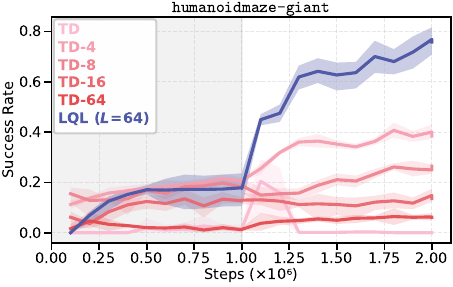}
    \caption{
        \textbf{LQL with long trajectories scales to the longest task in OGBench; $n$-step TD degrades as $n$ grows.}
        Sparse-reward \texttt{humanoidmaze-giant} (all tasks) with Best-of-$N$ policies. TD-$n$ at increasing $n$ and LQL with trajectory length $64$.
    }
    \label{fig:humanoidmaze_giant}
\end{wrapfigure}
Off-policy reinforcement learning holds the promise of turning past experience into future competence: by learning a value function, an agent can improve from data collected by older policies, other agents, or imperfect exploration, without requiring fresh on-policy rollouts at every update \citep{watkins_q-learning_1992,sutton_reinforcement_1998}. This promise is particularly compelling in domains where interaction is expensive and long-horizon behavior matters (e.g., robotics), where we would like to extract as much learning signal as possible from each transition. Classic value-based methods such as Q-learning do exactly this by \emph{bootstrapping}: they train a $Q$-function so that the utility of a state-action pair agrees with the immediate reward plus a discounted estimate of what the learned policy could achieve next, i.e., a counterfactual continuation under the policy being learned.

However, bootstrapping over long horizons is fragile in practice because estimation errors can amplify as they propagate backward through time \citep{jaakkola_convergence_1993,sutton_reinforcement_1998,asis_multi-step_2018,park_horizon_2025}.
A common mitigation is to reduce reliance on the next-state estimate by using multi-step targets (e.g., $n$-step or $\lambda$-returns), replacing a 1-step backup with a longer segment of observed rewards before switching back to a bootstrapped value estimate \citep{sutton_reinforcement_1998,hessel_rainbow_2017,asis_multi-step_2018}. Yet this remedy introduces its own tension in off-policy settings: an $n$-step target for $Q(s_t,a_t)$ necessarily incorporates the rewards generated by the \emph{logged} actions $a_{t+1},\ldots,a_{t+n-1}$ that followed in the replayed trajectory. As a result, even if $a_t$ is a high-quality decision, the resulting backup can be pessimistic when the subsequent recorded actions are low-quality.
Put differently, multi-step returns partially evaluate $(s_t,a_t)$ under the incorrect assumption that the agent will continue to follow the behavior that produced the trajectory for several more steps, rather than switching immediately to the (improving) policy being learned.

\begin{wrapfigure}{r}{0.45\textwidth}
    \vspace{-\intextsep}
    \centering
    \includegraphics[width=0.45\textwidth]{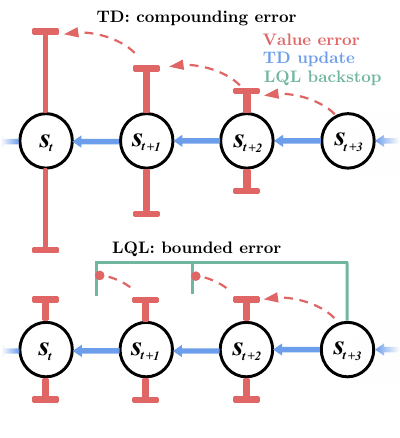}
    \caption{
        \textbf{LQL establishes a backstop against compounding TD error over time.} Standard 1-step TD can amplify estimation errors as they propagate backward through bootstrap updates (top). LQL's long-horizon constraints provide additional correction signals that bound these inconsistencies across multiple steps (bottom).
    }
\end{wrapfigure}
This paper asks: can we keep the simplicity and low-variance of 1-step off-policy TD learning, while introducing a principled long-horizon backstop that detects and corrects the kinds of inconsistencies that lead to compounding error? Our starting point is a standard consequence of optimality---and one previously leveraged for ``optimality tightening''---that any realized sequence of actions provides a lower bound on what the optimal policy could achieve in expectation from the same starting state \citep{he_learning_2016}. Concretely, for any trajectory of experience, the optimal value of the initial decision cannot be \emph{worse} (in expected return) than committing to the observed actions for a while and only then behaving optimally. We refer to this constraint as an \emph{optimality inequality} over trajectories.

Motivated by this inequality, we introduce a training objective that augments the standard TD error with \emph{hinge} penalty terms that enforce long-horizon consistency asymmetrically. Intuitively, if a value estimate at some state-action is \emph{below} the return realized by an actual action sequence, the estimate is pushed upward; conversely, if combining a later value estimate with the observed actions leading up to it would imply doing \emph{better} than acting optimally earlier, the later estimate is pushed downward. These hinge losses provide a trajectory-level correction signal, but can be computed from the same $Q$-values already used in standard TD updates.

Our main contribution is \emph{long-horizon Q-learning} (LQL), a general off-policy value learning approach that is agnostic to the policy extraction mechanism. LQL samples trajectories from the replay buffer and constructs hinge penalties using only network outputs already computed within a typical TD update (e.g., $Q(s_t,a_t)$ and the bootstrapping value at the next state), requiring no auxiliary networks and no additional forward passes per update compared to Q-learning. Empirically, across locomotion and robot manipulation tasks from OGBench \citep{park_ogbench_2025} and RoboMimic \citep{mandlekar_what_2021}, and across multiple ways of extracting a policy from the learned value function, LQL improves over TD learning while incurring minor additional runtime cost per training iteration.

\section{Related work}
\label{sec:related}

\textbf{Off-policy learning.}
Off-policy value-based RL (e.g., Q-learning) can in principle learn from arbitrary experience \citep{watkins_q-learning_1992,sutton_reinforcement_1998,mnih_human-level_2015}, but is often unstable in practice. This is commonly attributed to the \emph{deadly triad}---off-policy learning, bootstrapping, and function approximation---which causes TD errors to propagate and amplify over long horizons \citep{baird_residual_1995,tsitsiklis_analysis_1997,sutton_reinforcement_1998,vanhasselt2018deepreinforcementlearningdeadly,park_horizon_2025}. LQL addresses this issue by augmenting TD learning with a temporally extended \emph{consistency backstop} that penalizes value inconsistencies across temporally distant states.

\textbf{Multi-step backups.}
Multi-step backups, such as $n$-step returns and $\lambda$-returns, accelerate reward propagation and reduce sensitivity to compounding 1-step TD error \citep{sutton_learning_1988,watkins_q-lambda_1989,peng_incremental_1996,sutton_reinforcement_1998,mnih_asynchronous_2016,hessel_rainbow_2017,asis_multi-step_2018,schulman_high-dimensional_2018,chebotar_q-transformer_2023,schwarzer_bigger_2023,daley_averaging_2025}. However, they define targets from logged trajectories, coupling early actions to subsequent (possibly suboptimal) behavior. LQL instead uses trajectories to impose an \emph{optimality inequality}, regularizing long-horizon consistency while retaining standard TD updates.

\textbf{Off-policy corrections for multi-step learning.}
In off-policy settings, multi-step targets are biased under policy mismatch \citep{sutton_reinforcement_1998}. Importance sampling provides principled corrections \citep{precup_eligibility_2000} but suffers from high variance at long horizons \citep{hernandez-garcia_understanding_2019}, motivating truncated variants such as Retrace and V-trace \citep{munos_safe_2016,espeholt_impala_2018}. Many such methods require behavior and target action likelihoods \citep{sutton_reinforcement_1998}, which is inconvenient for expressive generative policies (e.g., diffusion or flow-matching) increasingly used in robotics \citep{ho_denoising_2020,song_score-based_2021,chi_diffusion_2023,ren_diffusion_2024,wagenmaker_steering_2025,park_flow_2025,black__05_2025,ai_joint_2026}. LQL is complementary: it derives hinge losses from an inequality characterization of the optimal $Q^\star$ and uses only existing TD network outputs, without action likelihoods or additional forward passes.

\textbf{Optimality tightening.}
LQL builds on the classical characterization of optimal value functions via Bellman (in)equalities \citep{bellman_dynamic_1957,puterman_markov_1994,sutton_reinforcement_1998}. A closely related method, \emph{optimality tightening} \citep{he_learning_2016}, also derives hinge losses from optimality inequalities. LQL differs by formulating the inequality using policy-generated actions, allowing reuse of the same network outputs as the TD loss and avoiding auxiliary networks and extra forward passes. In contrast, optimality tightening requires additional $Q$ evaluations per update (at least $2\times$, and $4\times$ in \citet{he_learning_2016}), which can substantially reduce wall-clock efficiency \citep{hessel_rainbow_2017,lee_sample-efficient_2019}.

\section{Preliminaries}
\label{sec:prelims}

We consider a Markov Decision Process (MDP) $\mathcal{M}=(\mathcal{S},\mathcal{A},P,R,\gamma)$ with discount $\gamma\in[0,1)$. A (stochastic) policy $\pi(\cdot\mid s)\in\Delta(\mathcal{A})$ induces trajectories $(s_0,a_0,r_0,s_1,\ldots)$ with $r_t=R(s_t,a_t)$ and $s_{t+1}\sim P(\cdot\mid s_t,a_t)$. The goal is to maximize the expected discounted return $\mathbb{E}_\pi[\sum_{t=0}^\infty \gamma^t r_t]$. The action-value function for $\pi$ is
$
Q^\pi(s,a)=\mathbb{E}_\pi[\sum_{t=0}^\infty \gamma^t r_t \mid s_0=s,a_0=a],
$
and the optimal action-value function $Q^*$ satisfies the Bellman optimality equation
\begin{equation}
\label{eq:bellman_opt}
Q^*(s,a)=\mathbb{E}_{s'\sim P(\cdot\mid s,a)}\!\left[ R(s,a)+\gamma \max_{a'} Q^*(s',a') \right].
\end{equation}

\textbf{Q-learning and target networks.}
In the function approximation setting, Q-learning trains a network $Q_\theta$ by minimizing a temporal-difference (TD) loss on replayed transitions $(s,a,r,s')$:
\begin{equation}
\label{eq:td_error_sample}
\ell_{\mathrm{TD}}(\theta)=\left(Q_\theta(s,a)-\big[r+\gamma \,Q_{\bar\theta}(s',a^*(s'))\big]\right)^2,
\end{equation}
where $Q_{\bar\theta}$ is a slowly-updated target network (e.g., soft updates $\bar\theta \leftarrow \tau \theta + (1-\tau)\bar\theta$ with $\tau\ll 1$), and $a^*(s')$ denotes the action used for bootstrapping. In discrete action spaces $a^*(s')=\arg\max_{a'} Q_{\bar\theta}(s',a')$; in continuous control we typically use a learned actor $\pi_\phi$ as a computational proxy for the maximizer, i.e., $a^*(s')=\pi_\phi(s')$.

\textbf{Trajectory notation.}
We will work with \emph{trajectories} of experience of length $L$ drawn from the replay buffer, $(s_{t:t+L},\, a_{t:t+L-1},\, r_{t:t+L-1})$, where $s_{t:t+L}$ denotes $(s_t,s_{t+1},\ldots,s_{t+L})$ and similarly for actions/rewards. For indices $i<j\le t+L$, define the discounted partial return along the observed trajectory as $G_{i:j}\triangleq\sum_{u=i}^{j-1}\gamma^{u-i} r_u$.

\section{Long-horizon Q-learning}
\label{sec:lql}

This section introduces \emph{long-horizon Q-learning} (LQL). The goal is to preserve the strengths of standard TD learning---transition-level, counterfactual bootstrapping---while adding a \emph{long-horizon backstop} that discourages the inconsistent value estimates responsible for compounding error. Concretely, we:
(i) derive trajectory-wise optimality inequalities,
(ii) interpret them as \emph{soft constraints} on the learned value function over replay-buffer trajectories, and
(iii) enforce these constraints using asymmetric (hinge) penalties that can be computed by \emph{reusing the same network evaluations} already required for TD learning.
We first present the constraint formulation, then derive the resulting penalty-based objective, and finally describe the practical sampled approximation and implementation details.

\subsection{Optimality inequalities over trajectories}
A consequence of optimality is that, from any state, acting optimally immediately cannot be worse (in expectation) than committing to an arbitrary sequence of actions for some duration and only then behaving optimally. One convenient form is: for any $i<j$,
\begin{equation}
\label{eq:optimality_ineq_trajectory}
Q^*(s_i,a_i)\;\ge\;\mathbb{E}\!\left[\,G_{i:j}+\gamma^{j-i} Q^*(s_j,a_j)\,\right],
\end{equation}
where the expectation is over the MDP dynamics (and any stochasticity in the action sequence, if applicable). We will use two variants that replace one side by an \emph{optimal action} to match the computation available in TD learning (for $j>i$):
\begin{align}
\label{eq:optimality_ineq_optLHS}
Q^*(s_i,a^*(s_i))
&\ge \mathbb{E}\!\Big[
G_{i:j}
+ \gamma^{j-i} Q^*(s_j,a_j)
\Big], \\
\label{eq:optimality_ineq_optRHS}
Q^*(s_i,a_i)
&\ge \mathbb{E}\!\Big[
G_{i:j}
+ \gamma^{j-i} Q^*(s_j,a^*(s_j))
\Big].
\end{align}
When $j=i+1$, Equation~\eqref{eq:optimality_ineq_optRHS} reduces to the Bellman optimality equation \eqref{eq:bellman_opt}.

\subsection{From optimality constraints to optimization objective}
\label{subsec:constrained_to_loss}
The inequalities in Equations \eqref{eq:optimality_ineq_optLHS}--\eqref{eq:optimality_ineq_optRHS} specify desirable properties of $Q^*$ over temporally distant states: the value of acting optimally earlier should upper bound the value implied by deferring optimality until later, and conversely, observed actions followed by later optimal behavior should lower bound what can be achieved when acting optimally earlier. A natural way to encode such inequalities in learning is via a constrained problem $\min_{\theta}\ \mathbb{E}[\ell_{\mathrm{TD}}(\theta)]$ subject to Equations~\eqref{eq:optimality_ineq_optLHS} and \eqref{eq:optimality_ineq_optRHS}. Rather than enforcing hard constraints, we use a standard penalty/Lagrangian-style relaxation: each inequality violation contributes a nonnegative penalty, yielding an unconstrained objective consisting of the TD loss plus weighted constraint penalties. Concretely, for each replayed trajectory we will create two families of penalties: \textbf{Lower-bound (LB) penalties} that push \emph{up} $Q_\theta(s_k,a_k)$ when it falls below a return realized by rolling forward along the trajectory and then bootstrapping with an (approximate) optimal action at a later state. \textbf{Upper-bound (UB) penalties} that push \emph{down} $Q_\theta(s_k,a_k)$ when combining it with preceding observed rewards would imply outperforming an optimal action taken earlier.

\subsection{Two-sided hinge penalties}
\label{subsec:hinge_penalties}

Consider a replay trajectory $(s_{t:t+L},a_{t:t+L-1},r_{t:t+L-1})$. For brevity, re-index positions within the trajectory as follows: for $k\in\{0,\ldots,L-1\}$, $(s_k,a_k,r_k)\triangleq(s_{t+k},a_{t+k},r_{t+k})$, and define $s_L\triangleq s_{t+L}$. Let $a_k^*\triangleq a^*(s_k)$ denote the bootstrap action at state $s_k$ (argmax in discrete actions, or $\pi_\phi(s_k)$ in continuous control). We use $G_{k:\ell}$ as defined above.

\paragraph{Lower-bound penalties.}
Using Equation~\eqref{eq:optimality_ineq_optRHS}, for any $\ell>k$ we would like
$
Q(s_k,a_k)\gtrsim G_{k:\ell} + \gamma^{\ell-k} Q(s_\ell,a_\ell^*).
$
Formally, Equation~\eqref{eq:optimality_ineq_optRHS} holds \emph{in expectation} over trajectories (and any randomness in the environment and policy). In practice, we approximate this expectation using a single sampled trajectory, i.e., we treat $G_{k:\ell}$ and $(s_\ell,a_\ell^*)$ from the replayed trajectory as a Monte Carlo estimate of the right-hand side.

We implement a soft version with a hinge-squared penalty where the bootstrap side uses the target network:
\begin{equation}
\label{eq:delta_LB}
\delta_{\mathrm{LB}}(k,\ell;\theta)
\;=\;
\Big[
G_{k:\ell}+\gamma^{\ell-k}Q_{\bar\theta}(s_\ell,a_\ell^*)
-\;Q_\theta(s_k,a_k)
\Big]_+^2.
\end{equation}
Intuitively, if the observed rewards plus a (target) bootstrap at time $\ell$ exceed the current estimate at time $k$, we push $Q_\theta(s_k,a_k)$ upward.
We use a hinge-squared form: the hinge makes the constraint one-sided (we only penalize violations), while squaring yields smoother gradients when violations occur. We aggregate over a set of future indices $\mathcal{F}(k)\subseteq\{k+1,\ldots,L\}$:
\begin{equation}
\label{eq:ell_LB}
\ell_{\mathrm{LB}}(k;\theta)
=
\frac{1}{|\mathcal{F}(k)|}\sum_{\ell\in\mathcal{F}(k)} \delta_{\mathrm{LB}}(k,\ell;\theta).
\end{equation}
A practical choice is to exclude $\ell=k+1$ since that case largely overlaps with the 1-step TD target; e.g., $\mathcal{F}(k)=\{k+2,\ldots,L\}$, and $\ell_{\mathrm{LB}}(k;\theta)=0$ if $\mathcal{F}(k)$ is empty.

\paragraph{Upper-bound penalties.}
Using Equation~\eqref{eq:optimality_ineq_optLHS}, for any $i \le k$ we would like
$
Q(s_i,a_i^*) \gtrsim G_{i:k} + \gamma^{k-i} Q(s_k,a_k).
$
Violations correspond to $Q(s_k,a_k)$ being too large relative to what could be achieved by acting optimally earlier, as measured by $Q(s_i,a_i^*)$. We penalize such violations with
\begin{equation}
\label{eq:delta_UB}
\delta_{\mathrm{UB}}(i,k;\theta)
\;=\;
\Big[
G_{i:k}+\gamma^{k-i}Q_\theta(s_k,a_k)
-\;Q_{\bar\theta}(s_i,a_i^*)
\Big]_+^2,
\end{equation}
where the optimal-action value on the right is evaluated by the target network and treated as a stable reference. Importantly, the special case $i=k$ yields a same-state upper bound: since $G_{k:k}=0$ and $\gamma^{0}=1$, the constraint reduces to $Q_{\bar\theta}(s_k,a_k^*) \ge Q_\theta(s_k,a_k)$, providing direct downward pressure when a suboptimal action is valued above the (approximate) greedy action at the same state. We aggregate upper-bound penalties for each $k$ by averaging over indices $i$ for which the target-network quantity $Q_{\bar\theta}(s_i,a_i^*)$ is already available from the TD computation on the same replay chunk. Concretely, since TD targets evaluate $Q_{\bar\theta}$ on next-states, we have target evaluations for states up to and including $s_k$ (via the immediately preceding transition), so we use $\mathcal{P}(k)\subseteq\{1,\ldots,k\}$, which includes the same-state upper bound term $i=k$ without requiring additional forward passes. We then define
\begin{equation}
\label{eq:ell_UB}
\ell_{\mathrm{UB}}(k;\theta)
=
\frac{1}{|\mathcal{P}(k)|}\sum_{i\in\mathcal{P}(k)} \delta_{\mathrm{UB}}(i,k;\theta),
\end{equation}
and set $\ell_{\mathrm{UB}}(k;\theta)=0$ if $\mathcal{P}(k)$ is empty (e.g., $k=0$). In contrast to our formulation of the upper bounds, \citet{he_learning_2016} formulates them using logged actions rather than actions from the learned policy at the earlier states, which precludes both the reuse of target-network outputs from earlier in the trajectory and the same-state upper bound. This difference in formulation allows us to use no additional forward passes in computing the hinge penalties.

\subsection{Final training objective}
\label{subsec:final_objective}

For each transition index $k\in\{0,\ldots,L-1\}$ in the trajectory, let $\ell_{\mathrm{TD}}(k;\theta)$ denote the 1-step TD loss \eqref{eq:td_error_sample} computed on $(s_k,a_k,r_k,s_{k+1})$. LQL augments TD learning with the long-horizon penalties above:
\begin{equation}
\label{eq:lql_loss_sample_revised}
\ell_{\mathrm{LQL}}(k;\theta)
=
\ell_{\mathrm{TD}}(k;\theta)
+\lambda_{\mathrm{UB}}\,\ell_{\mathrm{UB}}(k;\theta)
+\lambda_{\mathrm{LB}}\,\ell_{\mathrm{LB}}(k;\theta),
\end{equation}
with nonnegative weights $\lambda_{\mathrm{UB}},\lambda_{\mathrm{LB}}$. \emph{In this work, we do not tune these weights for the main baseline comparison experiments (e.g., Figure~\ref{fig:aggregate_success}) and use $\lambda_{\mathrm{UB}}=\lambda_{\mathrm{LB}}=1$}; we only vary them in the ablations in Figures ~\ref{fig:sample_ctrl} and ~\ref{fig:lambda_sweep}. We use the descriptor ``long-horizon'' in naming long-horizon Q-learning due to the fact that the hinge penalties $\ell_{\mathrm{UB}},\ell_{\mathrm{LB}}$ relate value estimates across temporal distances greater than one, in contrast to $\ell_{\mathrm{TD}}$. The overall loss for a sampled trajectory is the average over $k$: $\mathcal{L}_{\mathrm{LQL}}=\frac{1}{L}\sum_{k=0}^{L-1}\ell_{\mathrm{LQL}}(k;\theta)$.

\subsection{Practical notes}
\label{subsec:practical}

The inequalities in Equation~\eqref{eq:optimality_ineq_optLHS}--\eqref{eq:optimality_ineq_optRHS} are stated in expectation over dynamics. In LQL, we approximate these expectations using a sampled replay trajectory: the observed rewards within the trajectory provide an unbiased sample of the return prefix, while the remaining tail is approximated by a bootstrap term $Q_{\bar\theta}(s_\ell,a_\ell^*)$. Under deterministic transition dynamics, this sample-based approximation is unbiased; under stochastic dynamics, it introduces bias, which we explore further in Section~\ref{subsec:stochastic_envs} (Figure~\ref{fig:stochastic_envs}), Section~\ref{sec:conclusion}, and Appendix~\ref{app:false_penalties}. Crucially, the additional penalties in Equations \eqref{eq:delta_LB} and \eqref{eq:delta_UB} can be computed using the same forward passes already required by TD learning on the trajectory:
\begin{itemize}[topsep=0pt,itemsep=0pt,parsep=0pt]
\item We evaluate $Q_\theta(s_k,a_k)$ for all observed state-action pairs in the trajectory (needed for $\ell_{\mathrm{TD}}$).
\item We evaluate $Q_{\bar\theta}(s_k,a_k^*)$ for bootstrap actions at the corresponding states (needed for TD targets).
\end{itemize}
No auxiliary networks are introduced, and we do not require additional forward passes beyond those used to compute TD losses along the trajectory. In implementation, we stop gradients through all target-network quantities $Q_{\bar\theta}(\cdot,\cdot)$, and treat $a_k^*$ as a bootstrap action (computed via argmax or sampling from the actor) exactly as in standard off-policy TD learning. Finally, we compute the discounted partial returns $G_{i:j}$ efficiently using prefix sums along the trajectory, enabling evaluation of all selected pairs $(i,k)$ and $(k,\ell)$ with negligible additional overhead.

\textbf{Compatibility with existing value learning methods.} We derive the above results by considering the augmentation of 1-step TD learning, in which case we establish hinge penalties that apply across transitions with a minimum separation of one time step---ultimately, for simplicity. However, these hinge penalties can also be applied to $n$-step TD learning and action chunking ~\citep{li_reinforcement_2025}, where the hinge penalties would be established over a minimum separation of $n$ (or the action chunk size) time steps; we also experimentally evaluate LQL on action chunking below.

\begin{algorithm}[tb]
\caption{LQL update on a replay trajectory}
\label{alg:lql_update}
\begin{algorithmic}
\STATE \textbf{Input:} replay buffer $\mathcal{D}$; online $Q_\theta$; target $Q_{\bar\theta}$; bootstrap policy/actor $\pi_\phi$ (or argmax); step size $\eta$
\STATE Sample a trajectory $(s_{0:L},a_{0:L-1},r_{0:L-1}) \sim \mathcal{D}$
\STATE Compute bootstrap actions $a_k^* \leftarrow \pi_\phi(s_k)$ for $k=0,\ldots,L$ \;\; (or $a_k^*=\arg\max_a Q_{\bar\theta}(s_k,a)$)
\STATE Evaluate $Q_\theta(s_k,a_k)$ for $k=0,\ldots,L-1$
\STATE Evaluate $Q_{\bar\theta}(s_k,a_k^*)$ for $k=0,\ldots,L$ \\
$\mathcal{L}_{\mathrm{LQL}}=\frac{1}{L}\sum_{k=0}^{L-1}\Big(\ell_{\mathrm{TD}}(k;\theta)+\lambda_{\mathrm{UB}}\ell_{\mathrm{UB}}(k;\theta)+\lambda_{\mathrm{LB}}\ell_{\mathrm{LB}}(k;\theta)\Big)$
\STATE $\theta \leftarrow \theta - \eta \nabla_\theta \mathcal{L}_{\mathrm{LQL}}$
\STATE Update target network (e.g., $\bar\theta\leftarrow \tau\theta+(1-\tau)\bar\theta$)
\end{algorithmic}
\end{algorithm}

\section{Experiments}
\label{sec:experiments}

We test when and why LQL helps:

\begin{enumerate}[topsep=0pt,itemsep=0pt,parsep=0pt,leftmargin=*]
    \item Across diverse
    offline-to-online tasks and policy classes, how does LQL compare to
    1-step TD, length-matched $n$-step TD, and the closest prior
    optimality-inequality method, OT~\citep{he_learning_2016}?
    \item Does trajectory length offer a useful
    scaling axis, particularly where TD batch-size scaling is known to
    fail~\citep{fu_compute-optimal_2025}?
    \item Does LQL retain its benefit under
    stochastic dynamics, where the in-expectation inequalities
    introduce a theoretical bias?
    \item Are gains attributable to the hinge
    penalties themselves rather than to side effects of trajectory
    sampling, and do they leave a measurable signature on
    value-function stability?
\end{enumerate}

We apply LQL to two of its three compatible settings---1-step TD and
action chunking---leaving the $n$-step TD case to future work.

\subsection{Environments and datasets}
\label{subsec:envs}

We evaluate on robot manipulation and navigation tasks from
OGBench~\citep{park_ogbench_2025} and RoboMimic~\citep{mandlekar_what_2021},
which have several favorable properties for this
evaluation (more details in Appendix~\ref{app:tasks}): \textbf{Long horizons:} In OGBench's \texttt{humanoidmaze-giant}, a 21-DoF humanoid must traverse mazes that take thousands of environment steps to solve, with a single sparse reward upon reaching the goal. 1-step TD must therefore propagate signal across an extreme bootstrap chain, which is the failure mode LQL targets. \textbf{Behavioral heterogeneity:} In RoboMimic, we use multi-human datasets in which trajectories were collected by operators of varying proficiency. This is a regime where $n$-step targets in principle suffer from the contribution of low-quality logged actions. \textbf{Stitching dependence:} In OGBench, \texttt{humanoidmaze} and \texttt{cube-triple} use the \texttt{navigate}-style and \texttt{play}-style datasets respectively, which are collected by exploratory rather than task-directed policies. This makes stitching together trajectories from these exploratory policies important, which is also a known theoretical limitation of TD-$n$~\citep{precup_eligibility_2000,munos_safe_2016}. We use two standard protocols
(Appendix~\ref{app:offline_online_protocols}): RLPD-style continual
mixing for the Gaussian actor, and two-stage offline$\rightarrow$online
training for the others. Behavior-regularization coefficients (e.g.,
$\alpha$ in FQL/QC-FQL) are identical for LQL, TD, and TD-$n$
(Table~\ref{table:alpha_sets}).

\subsection{Comparison methods and protocol}
\label{subsec:comparisons}

To isolate the effect of the value-learning rule, we train Q-functions
using \textbf{(a) 1-step TD}, \textbf{(b) length-matched $n$-step TD}
($n=L=8$), and \textbf{(c) LQL}, then apply identical policy extraction
mechanisms to each learned value function. We treat 1-step TD as a
baseline to isolate the contribution of LQL's backstop beyond standard
bootstrapped learning and length-matched $n$-step TD as an existing
recipe for incorporating multi-step information into off-policy value
learning. For breadth, we also report
\textbf{IQL}~\citep{kostrikov_offline_2021} and
\textbf{ReBRAC}~\citep{tarasov_revisiting_2023}. We compare
directly to optimality tightening~\citep{he_learning_2016} in
Section~\ref{subsec:ot_comparison}.

\textbf{Compute parity.} For all critic-learning rules we match the total
number of Q-network forward passes per gradient step. With a batch size
of $1024$ transitions and LQL trajectory length $8$, LQL samples $128$
starting indices uniformly from the replay buffer and unrolls length-$8$
segments; TD and $n$-step TD use the same effective number of transitions
per step. More details are in Appendix~\ref{app:impl_details}.

\textbf{Policy extraction families.} To test that improvements from the
critic carry across actor parameterizations, we evaluate four policy
extraction mechanisms representative of contemporary offline-to-online RL:
a \textbf{Gaussian actor} (RLPD-style~\citep{ball_efficient_2023}),
\textbf{Best-of-$N$ (BFN)} sampling from a flow-matching behavior
policy~\citep{stiennon_learning_2022}, \textbf{Flow Q-learning
(FQL)}~\citep{park_flow_2025}, and an \textbf{action-chunked policy} via
the QC-FQL recipe of~\citet{li_reinforcement_2025}. Action chunking is
particularly important because it underlies many state-of-the-art
real-world robot policies~\citep{zhao_learning_2023,amin__06_2025}.

\begin{figure}[t]
    \centering
    \includegraphics[width=\linewidth]{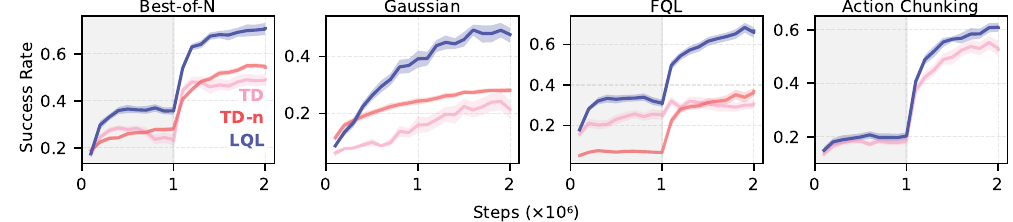}
    \caption{
        \textbf{Across policy-extraction families, LQL achieves higher average success than TD and TD-$n$.} Mean success rate averaged over all environments, separated by policy type. Background shading indicates whether the update includes an environment interaction step (white) or is purely offline (gray).
    }
    \label{fig:aggregate_success}
    \vspace{-3mm}
\end{figure}

\begin{figure}[t]
    \centering
    \includegraphics[width=\linewidth]{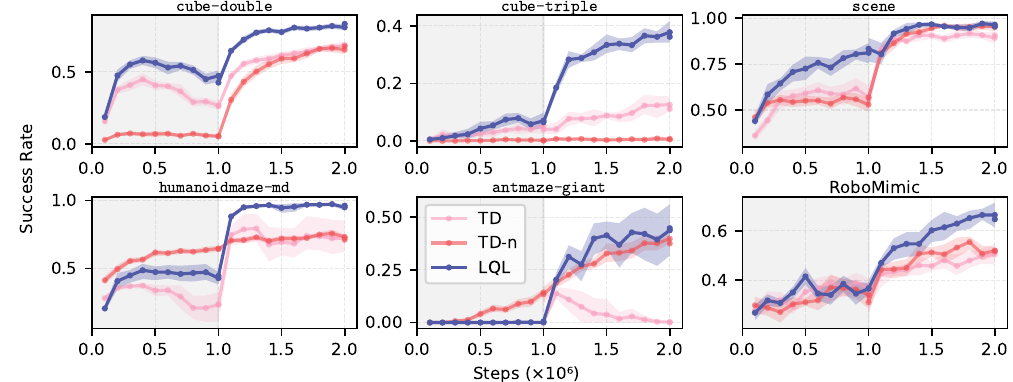}
    \caption{
        \textbf{For Best-of-$N$ policies, LQL improves over both 1-step TD and TD-$n$ across task groups.} Each panel aggregates success rates within a task group over training.
    }
    \vspace{-3mm}
    \label{fig:aggregate_success_bfn}
\end{figure}

\subsection{Comparative evaluation across tasks and policy classes}
\label{subsec:main_results}

Per-task and per-suite results appear in
Tables~\ref{table:OGBench_success} and~\ref{table:robomimic_success};
Best-of-$N$ aggregates are in Figure~\ref{fig:aggregate_success_bfn}, and
the cross-policy aggregation is in Figure~\ref{fig:aggregate_success}.
Across all four policy extraction mechanisms, LQL improves average
success rate over both 1-step TD and length-matched $n$-step TD. LQL's
advantage over $n$-step TD indicates that its improvement is not just
from using $n$-step returns. Length-matched $n$-step TD
often provides only a modest improvement over 1-step TD, and on several
tasks (e.g., \texttt{cube-triple} with FQL) it performs worse, consistent
with the well-known off-policy bias of $n$-step
targets~\citep{precup_eligibility_2000,munos_safe_2016}. LQL's more
favorable usage of the multi-step target is most visible on tasks
(\texttt{humanoidmaze-md}, \texttt{cube-triple}) where success is most
reliant on stitching together segments of independently suboptimal
offline experience.

\textbf{Compute overhead.} These gains come with a small per-update
slowdown of $4.7\%$ on average across policy classes
(Table~\ref{table:runtimes}, Appendix~\ref{app:compute_requirements}),
arising from computing pairwise discounted returns within each sampled
trajectory. This overhead does not scale with critic or actor size, so
its relative cost shrinks as networks grow. We also observe negligible
costs for $L=64$ (Appendix~\ref{app:compute_requirements}).

\subsection{Trajectory length as an axis for scaling}
\label{subsec:trajectory_length}

\begin{wrapfigure}{r}{0.5\textwidth}
    \vspace{-3\baselineskip}
    \centering
    \includegraphics[width=\linewidth]{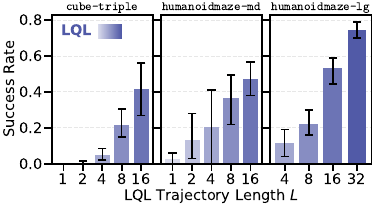}
    \caption{
        \textbf{LQL performance improves as trajectory length $L$ grows.}
        Each panel sweeps $L$ at fixed trajectories per batch, so larger $L$
        means strictly more compute per step. \emph{Left, middle:}
        FQL actor, 128 trajectories per batch, on \texttt{task3} and
        \texttt{task4} respectively. \emph{Right:} Best-of-$N$ actor,
        64 trajectories per batch, on \texttt{task3}. Network sizes are held
        fixed. Increasing $L$ raises final
        success rate in all three settings.
    }
    \label{fig:horizon_scale}
\end{wrapfigure}

A practical concern with deep value-based RL is that, unlike supervised
learning, increasing the batch size for TD does not reliably improve
performance and frequently degrades it past a
threshold~\citep{fu_compute-optimal_2025} (we reproduce this in
Figure~\ref{fig:horizon_batch_scale}). This is hypothesized to reflect
overfitting to inaccurate target estimates. We investigate whether
LQL's trajectory length offers an alternative scaling axis that does not
suffer this pathology. Figure~\ref{fig:horizon_scale} sweeps trajectory length $L$ at fixed number of trajectories per batch (so larger $L$ means more transitions
per update). In all cases, LQL's final success rate
improves with $L$ over the swept range. This contrasts with the
corresponding TD batch-size sweep, where additional transitions yield
no improvement or hurt performance (Figure~\ref{fig:horizon_batch_scale}).

We stress-test this in the longest task in OGBench, sparse-reward
\texttt{humanoidmaze-giant}, with $L=64$
(Figure~\ref{fig:humanoidmaze_giant}, Table~\ref{table:batch_scale_giant}).
Here 1-step TD never solves a single task ($0\%$ across all five tasks).
$n$-step TD partially helps but plateaus at $n=4$
($38.4\%$ averaged) and \emph{degrades} as $n$ grows further, with
$n=64$ achieving only $6.1\%$. LQL with $L=64$ achieves $75.7\%$,
saturating two tasks ($97.3\%$ and $98.7\%$). We read this as direct evidence that the hinge constraints,
unlike multi-step TD targets, can absorb information from very long
trajectories without inheriting the off-policy bias that causes $n$-step
TD to deteriorate at large $n$.

\subsection{Further experiments}
\label{subsec:further_experiments}

\textbf{Mechanism: hinge penalties vs.\ trajectory sampling.}
\label{subsec:mechanism}
LQL differs from 1-step TD in two ways simultaneously: it samples short
trajectories rather than independent transitions, and it adds
hinge penalties on top of the TD loss. To attribute the gains correctly,
we ablate $\lambda_{\mathrm{LB}}=\lambda_{\mathrm{UB}}=0$, which preserves
the trajectory sampling but removes the long-horizon constraints, leaving
1-step TD trained on trajectory-sampled minibatches
(Figure~\ref{fig:sample_ctrl}, Appendix~\ref{app:sample_ctrl}).
Removing the hinge penalties substantially reduces performance, verifying
their importance. Interestingly, we see that trajectory sampling
sometimes outperforms standard transition-level TD, which we discuss
further in Appendix~\ref{app:sample_ctrl}.

\textbf{Robustness to environment stochasticity.}
\label{subsec:stochastic_envs}
We next ask whether the theoretical bias induced by environmental
stochasticity materially harms practical performance.
We construct stochastic versions of OGBench tasks by adding zero-mean
Gaussian noise with standard deviation $\sigma$ to actions before
execution and recollecting matched offline datasets at each $\sigma$
(Appendix~\ref{app:stochastic_envs_details}). Across all measured $\sigma$,
LQL matches or outperforms 1-step TD (Figure~\ref{fig:stochastic_envs}).

\textbf{Sensitivity to hinge coefficient.}
\label{subsec:lambda_sweep}
All comparative results above use
$\lambda_{\mathrm{LB}} = \lambda_{\mathrm{UB}} = 1$, chosen as a simple
default rather than tuned per task. Figure~\ref{fig:lambda_sweep} sweeps
these coefficients while holding all other hyperparameters fixed.
Performance is broadly stable across roughly an order of magnitude
around the default.

\begin{wrapfigure}{r}{0.5\textwidth}
\vspace{-\intextsep}
    \centering
    \vspace{2\baselineskip}
    \includegraphics[width=\linewidth]{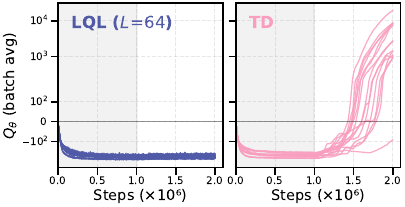}
    \caption{
        \textbf{LQL keeps $Q$-values within the analytically valid range;
        1-step TD diverges.}
        Average online $Q_\theta(s,a)$ during training on \texttt{humanoidmaze-giant} (rewards in $\{-1,0\}$),
        one curve per task/seed. Since rewards are non-positive, $Q^*\le 0$
        everywhere.
    }
    \label{fig:q_blowup_giant}
\end{wrapfigure}

\textbf{$Q$-value stability.}
\label{subsec:why_lql_works}
We also analyzed the learned value function for TD and LQL directly.
Figure~\ref{fig:q_blowup_giant} plots the average online $Q$-value
during training on \texttt{humanoidmaze-giant}, where the reward is in
$\{-1, 0\}$, so $Q^*$ must be non-positive everywhere. For 1-step TD,
14/15 training runs blow up, with the average $Q$ crossing zero
and growing to magnitudes exceeding $900$. With LQL, $Q$-values remain in
the analytically valid range across all tasks. This is direct evidence
of the hinge upper-bound penalty preventing the common runaway overestimation
of TD.

\textbf{Comparison to optimality tightening.}
\label{subsec:ot_comparison}
LQL's most similar prior work, OT~\citep{he_learning_2016}, differs from
our implementation as described in Sections~\ref{sec:related}
and~\ref{sec:lql}. We adapt the public OT
codebase to run on our evaluations and show in
Appendix~\ref{app:ot_cmp} that LQL outperforms OT, supporting the design
choices that distinguish the two methods.

\section{Conclusion}
\label{sec:conclusion}

We introduced \emph{long-horizon Q-learning} (LQL), a practical
algorithm for mitigating compounding TD error in off-policy value
learning. Starting from the optimality inequality---a principled
relationship that constrains value estimates separated by any time
horizon---we derived two-sided hinge penalties that reuse the network
outputs already produced by standard TD updates, requiring no auxiliary
networks and no additional forward passes. Across locomotion and
manipulation tasks from OGBench and RoboMimic and across four
policy-extraction families, LQL consistently improves over both 1-step
TD and $n$-step TD at a small runtime cost. Most strikingly, on the sparse-reward \texttt{humanoidmaze-giant}---the longest task in OGBench---LQL achieves
$75.7\%$ success across the five tasks while 1-step TD fails to solve
any of them and
$n$-step TD plateaus at $n{=}4$ and degrades thereafter. These results carry two broader implications: LQL provides
a path for absorbing information from very long trajectories without
inheriting the off-policy bias that limits $n$-step targets
\citep{precup_eligibility_2000,munos_safe_2016}, and trajectory length
emerges as a scaling axis that does not exhibit the pathological degradation of batch scaling in TD \citep{fu_compute-optimal_2025}.

The main limitation of this work is theoretical: the optimality
inequality is an \emph{in-expectation} statement, so per-sample
penalties introduce a bias in stochastic environments. We give a
horizon-independent bound on this bias in
Appendix~\ref{app:false_penalties}; the bound tightens as the
suboptimality of the behavior data grows, and
Figure~\ref{fig:stochastic_envs} shows LQL matching or outperforming
1-step TD across all noise levels we tested. Several directions follow
naturally: applying the same hinge framework on top of $n$-step TD,
which could combine LQL's stability with $n$-step TD's faster reward
propagation; a more mechanistic study of how the upper- and lower-bound
terms each contribute to the $Q$-value stability we observe in
Figure~\ref{fig:q_blowup_giant}; and pushing trajectory length beyond
$L=64$, where our monotonic sweep suggests substantial headroom
remains.

\section{Acknowledgments}
This work was supported by the Robotics and AI Institute (RAI) and ONR grant N00014-22-1-2621.

\newpage
\appendix
\FloatBarrier
\section{Experiment Details}

\subsection{Evaluation Protocol}
\label{app:eval_protocol}

For all methods, we run 4 seeds for each \texttt{task} in each OGBench task group (e.g., \texttt{cube-double-play-singletask-task1-v0}) and 4 seeds for each task in RoboMimic (e.g., \texttt{can}). Throughout the paper, confidence intervals are 95\% intervals from a curve-level bootstrap (1000 iterations) over training runs; for multi-task aggregates, we resample runs within each task following \citet{agarwal_deep_2021}.

\subsection{Tasks}
\label{app:tasks}

\begin{figure}[H]
  \centering
  \begin{tabular}{cccc}
    \includegraphics[width=0.21\linewidth]{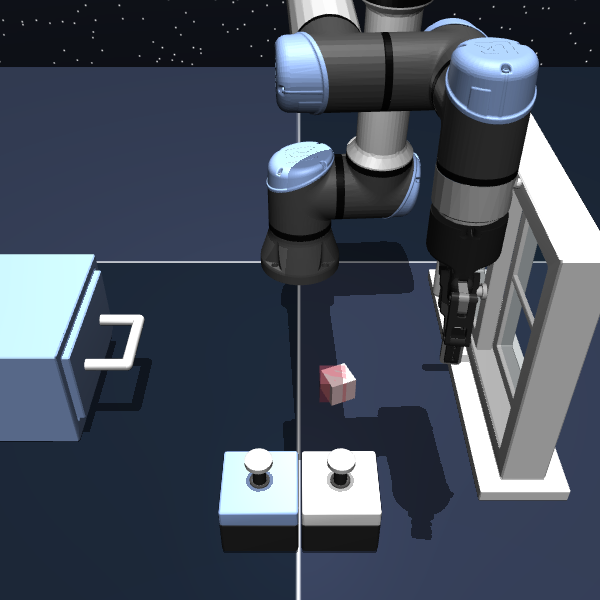} &
    \includegraphics[width=0.21\linewidth]{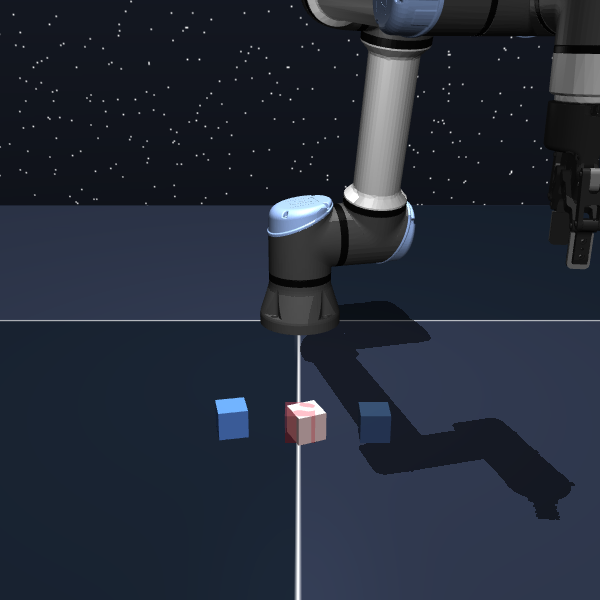} &
    \includegraphics[width=0.21\linewidth]{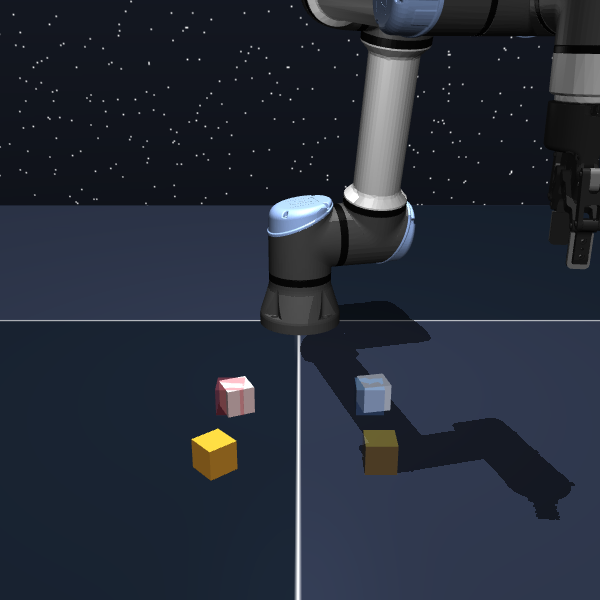} &
    \includegraphics[width=0.21\linewidth]{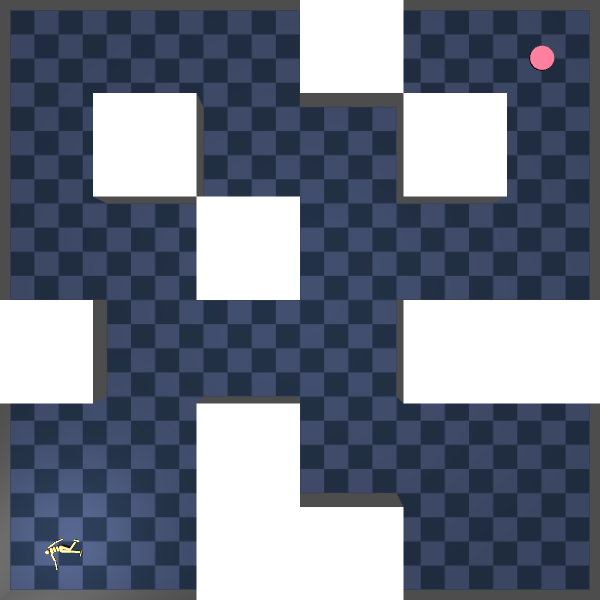} \\
    {\small (a) \texttt{scene}} & {\small (b) \texttt{cube-double}} & {\small (c) \texttt{cube-triple}} & {\small (d) \texttt{humanoidmaze-md}} \\[4pt]
    \includegraphics[width=0.21\linewidth]{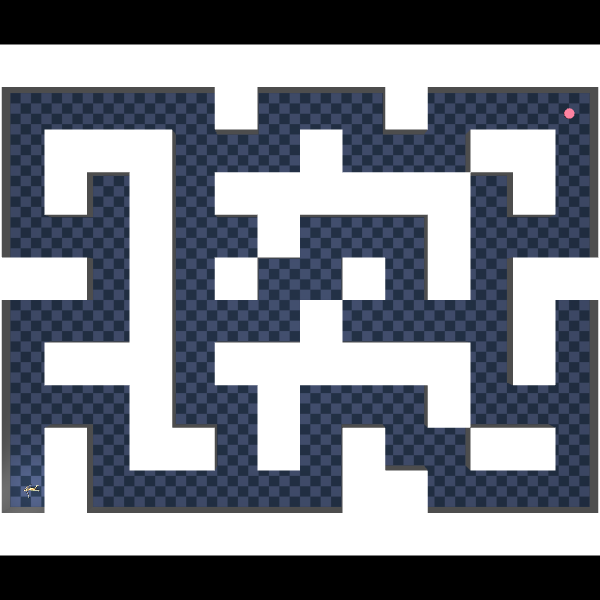} &
    \includegraphics[width=0.21\linewidth]{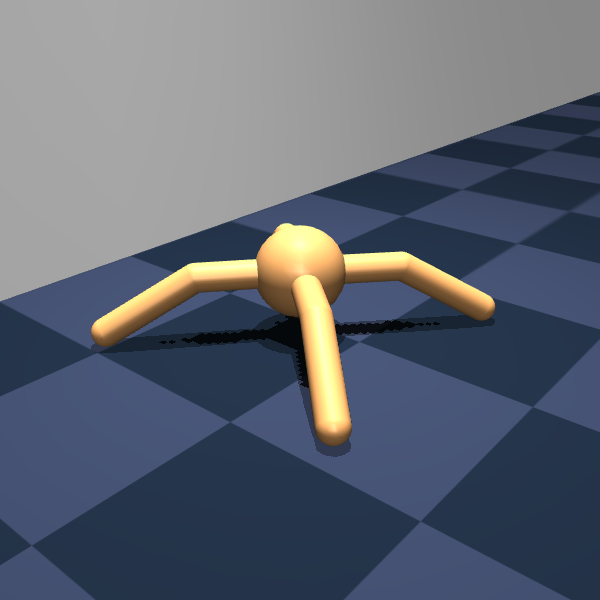} &
    \includegraphics[width=0.21\linewidth]{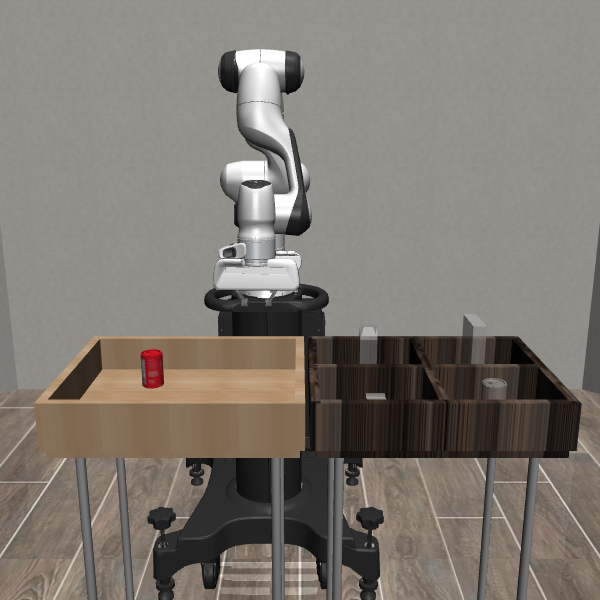} &
    \includegraphics[width=0.21\linewidth]{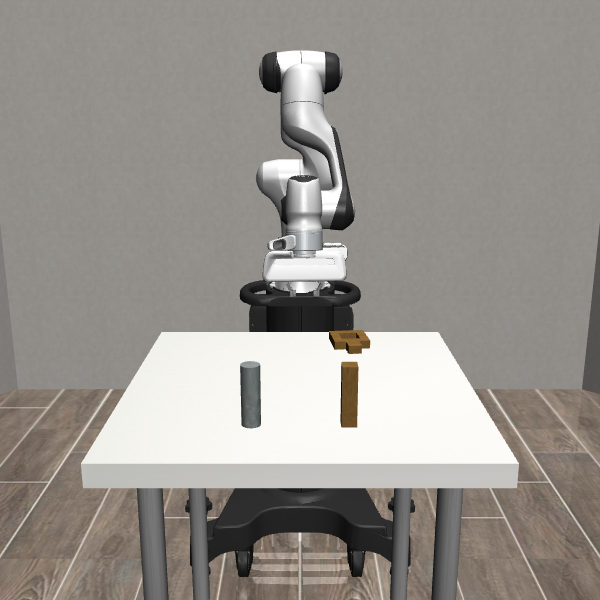} \\
    {\small (e) \texttt{humanoidmaze-giant}} & {\small (f) \texttt{antmaze-giant}} & {\small (g) \texttt{can}} & {\small (h) \texttt{square}} \\
  \end{tabular}
  \caption{Task panel for the OGBench~\citep{park_ogbench_2025} and RoboMimic~\citep{mandlekar_what_2021} domains we evaluate on. The \texttt{antmaze-giant} maze layout matches \texttt{humanoidmaze-giant}.}
  \label{fig:task_panel}
\end{figure}

\textbf{OGBench tasks.} We use the single-task variants of OGBench, which fix an evaluation goal and relabel the dataset with a corresponding (semi-)sparse reward.
\begin{itemize}
  \item \texttt{scene}: A robot arm interacts with a drawer, a window, a cube, and two button locks that gate the drawer and window. Each task requires a sequence of manipulations such as unlocking, opening, placing, and closing. We modify the reward to be sparse, similar to~\citet{li_reinforcement_2025}.
  \item \texttt{cube-double}, \texttt{cube-triple}: A robot arm rearranges 2 or 3 cubes to specified target configurations.
  \item \texttt{humanoidmaze-md}, \texttt{humanoidmaze-giant}: A 21-DoF Humanoid agent navigates a maze. The \texttt{giant} maze contains paths that can require thousands of environment steps to traverse. For \texttt{humanoidmaze-giant}, we use the 100M-transition offline dataset released by the original authors subsequent to the OGBench publication. Sparse reward.
  \item \texttt{antmaze-giant}: An 8-DoF quadrupedal Ant agent navigates the giant maze. Sparse reward.
\end{itemize}

\textbf{RoboMimic tasks.} We use the multi-human demonstration datasets, which contain trajectories collected by multiple human operators of varying proficiency.
\begin{itemize}
  \item \texttt{can}: A robot arm picks a soda can and places it in a target bin.
  \item \texttt{square}: A robot arm picks a square nut and places it on a rod.
\end{itemize}

\begin{table}[H]
  \caption{Task characteristics. Episode length refers to the maximum number of environment steps per episode.}
  \centering
  \begin{small}
    \begin{tabular}{cccc}
      \toprule
      \textbf{Task group} & \textbf{Dataset size (transitions)} & \textbf{Episode length} & \textbf{Action dimension} \\
      \midrule
      \texttt{scene} & 1M & 750 & 5 \\
      \texttt{cube-double} & 1M & 500 & 5 \\
      \texttt{cube-triple} & 3M & 1000 & 5 \\
      \texttt{humanoidmaze-md} & 2M & 2000 & 21 \\
      \texttt{humanoidmaze-giant} & 100M & 4000 & 21 \\
      \texttt{antmaze-giant} & 1M & 1000 & 8 \\
      \texttt{can} & 62756 & 500 & 7 \\
      \texttt{square} & 80731 & 500 & 7 \\
      \bottomrule
    \end{tabular}
    \label{table:task_characteristics}
  \end{small}
\end{table}

\subsection{Implementation details}
\label{app:impl_details}

Here, we outline a few more implementation details of LQL. The surrounding training infrastructure is described in Appendix \ref{app:offline_online_protocols}, so we focus here on implementation details following the receipt of a shape \texttt{(batch, trajectory)} array of transitions. For learning single-action policies, the array we receive contains $128$ trajectories of $8$ transitions. The loss in Equation \ref{eq:lql_loss_sample_revised} is then averaged over each trajectory, and an actor loss specific to each policy class is computed from all $128\times8$ observations and, for policies involving behavioral cloning, corresponding realized actions. For action chunking policies, we enforce our consistent length-$8$ trajectory at the action chunk level, meaning we sample \texttt{(batch, trajectory * chunk\_size)=(128, 8*5)=(128, 40)} transitions. In the same way as non-LQL action chunk Q learning is performed, we then evaluate the value network with the first observation of each chunk and all of the actions in the chunk as input, which yields $8$ value function evaluations per trajectory, equivalently to LQL applied to single-action value learning. All \texttt{(128, 8)} action chunks are then passed to an \texttt{actor\_loss} function, which performs an FQL \citep{park_flow_2025} actor update on each chunk.

\subsection{Policy extraction details}
\label{app:policy_extraction}

The four policy extraction families described in Section~\ref{subsec:comparisons} share the same critic learned by TD, TD-$n$, or LQL, but differ in how actions are selected from the value function. Here we summarize each method and note any deviations from the original recipes.

\textbf{Best-of-$N$ (BFN).}
We train a flow-matching behavior policy $\mu_\theta(s, z): \mathcal{S}\times\mathbb{R}^{\dim(\mathcal{A})}\to\mathcal{A}$ with the standard flow-matching behavioral cloning loss, as in \citet{park_flow_2025}, integrating the flow ODE with the Euler method using $10$ steps. At action-selection time, we sample $N=16$ noise vectors $z\sim\mathcal{N}(0, I)$, push each through the flow to obtain candidate actions, and choose the one that maximizes the learned $Q$ \citep{stiennon_learning_2022}. No reparameterized policy gradient is taken through the flow, so the flow is trained only to match the data distribution and the $Q$-function performs all of the policy improvement.

\textbf{Flow Q-learning (FQL).}
FQL \citep{park_flow_2025} keeps the same BC flow $\mu_\theta$ as above and additionally trains a separate one-step policy $\mu_\omega(s, z): \mathcal{S}\times\mathbb{R}^{\dim(\mathcal{A})}\to\mathcal{A}$ that maps a noise sample directly to an action with a single network forward pass. The one-step policy is trained with
\[
\mathcal{L}_\pi(\omega) = \mathbb{E}_{s, z}\!\left[-Q_\theta(s, \mu_\omega(s, z))\right] + \alpha\, \mathbb{E}_{s, z}\!\left[\|\mu_\omega(s, z)-\mu_\theta(s, z)\|_2^2\right],
\]
where the second term distills the multi-step BC flow into the one-step policy. \citet{park_flow_2025} show that this distillation loss is an upper bound on the squared 2-Wasserstein distance between the two implied action distributions, so $\alpha$ acts as a behavioral cloning coefficient that controls how tightly the one-step policy stays near the BC flow; we use the values listed in Table~\ref{table:alpha_sets}. At deployment, action selection uses the one-step policy directly.

\textbf{Gaussian (entropy-regularized).}
Our Gaussian actor follows the soft actor-critic (SAC) formulation \citep{haarnoja_soft_2018} combined with the offline-data design choices from RLPD \citep{ball_efficient_2023}. The actor outputs a tanh-squashed diagonal Gaussian $\pi_\phi(\cdot\mid s)$ and is trained with the standard maximum-entropy objective:
\[
\mathcal{L}_\pi(\phi) = \mathbb{E}_s\!\left[-Q_\theta(s, a^\pi) + \alpha_{\mathrm{ent}}\log\pi_\phi(a^\pi\mid s)\right],
\]
where $a^\pi\sim\pi_\phi(\cdot\mid s)$ is sampled by the standard reparameterization trick. The temperature $\alpha_{\mathrm{ent}}$ is automatically tuned against a target entropy of $-0.5\dim(\mathcal{A})$ via the standard dual update on $\alpha_{\mathrm{ent}}$. From RLPD we adopt symmetric sampling, in which each minibatch is split equally between transitions sampled from the offline dataset and transitions sampled from the online replay buffer. We deviate from the original RLPD recipe in two places: we use $2$ critics with LayerNorm rather than the $10$-critic ensemble used in the original paper, and we omit the entropy backup in the critic target (Table~\ref{table:hyperparams}).

\textbf{Action-chunked (QC-FQL-style).}
For the action-chunked policy we follow the recipe of \citet{li_reinforcement_2025}. Both the critic and actor operate over action chunks $a_{t:t+h}$ of length $h=5$: the critic is $Q_\theta(s_t, a_{t:t+h})$, and the actor uses the FQL parameterization above to produce the entire chunk in a single forward pass from a single noise sample. Because the critic conditions on the full executed action sequence, the corresponding $h$-step return is unbiased even when the data are off-policy \citep{li_reinforcement_2025}. The actor update is identical to FQL applied at the chunk level, again with $\alpha$ from Table~\ref{table:alpha_sets}, and the LQL hinge penalties are imposed across action chunks rather than individual transitions, as described in Appendix~\ref{app:impl_details}.

\subsection{Hyperparameters}
\label{app:hyperparams}

\begin{table}[H]
  \vspace{-1\baselineskip}
  \caption{Common hyperparameters.}
  \centering
  \begin{small}
    \begin{tabular}{cc}
      \toprule
      \textbf{Parameter} & \textbf{Value} \\
      \midrule
      Batch size (number of transitions) & 1024 \\
      LQL trajectories per batch & 128 \\
      LQL trajectory length & 8 \\
      Discount factor & 0.99 \\
      Optimizer & Adam \citep{kingma_adam_2017} \\
      Learning rate & $3 \times 10^{-4}$ \\
      Target network update rate & $5 \times 10^{-3}$ \\
      Critic ensemble size & 2 \\
      UTD Ratio & 1 \\
      Number of flow steps & 10 \\
      Number of samples in Best-of-$N$ sampling & 16 \\
      Number of training steps & $2 \times10^6$ \\
      Number of online training steps & $2 \times10^6$ for RLPD-based approaches, $1 \times 10^6$ otherwise \\
      Network width & 512 \\
      Network depth & 4 \\
      Activation function & GELU \citep{hendrycks_gaussian_2023} \\
      Actor layer norm & No \\
      Critic layer norm & Yes \\
      Entropy backup in RLPD & No \\
      Action chunk size & $5$ \\
      LQL $\lambda_{\mathrm{UB}}, \lambda_{\mathrm{LB}}$ & $1$ (except in Figures~\ref{fig:sample_ctrl} and~\ref{fig:lambda_sweep}) \\
      Target Q aggregation & OGBench: mean, RoboMimic: min \\
      \bottomrule
    \end{tabular}
    \label{table:hyperparams}
  \end{small}
\end{table}

\textbf{\texttt{humanoidmaze-giant}-specific hyperparameters.} We use the 100M-transition navigate-style dataset released by the OGBench~\citep{park_ogbench_2025} authors, and a discount factor of $0.995$. We use the same number of trajectories per batch as in the main experiments (Figure~\ref{fig:aggregate_success_bfn}, $128$), but with a longer trajectory length $L=64$. Similar to the earlier experiments (Table~\ref{table:hyperparams}), we match the batch size of each TD variant to LQL at the transition level: $8192$ ($128\times64$). Besides this, all other configuration is identical to Table~\ref{table:hyperparams}, including network architectures and the use of Best-of-$N$ actors with $N=16$.

\begin{table}[H]
  \centering
  \caption{$\alpha$ values by task group, used for both FQL and QC.}
  \begin{small}
    \begin{tabular}{lc}
      \toprule
      Task group & $\alpha$ \\
      \midrule
      \texttt{humanoidmaze-md} & $10$ \\
      \texttt{antmaze-giant} & $5$ \\
      \texttt{cube-double}, \texttt{cube-triple}, \texttt{scene} & $100$ \\
      RoboMimic (\texttt{can}, \texttt{square}) & $250$ \\
      \bottomrule
    \end{tabular}
    \label{table:alpha_sets}
  \end{small}
\end{table}

For ReBRAC and IQL, we reuse results from \citet{li_reinforcement_2025} for all OGBench tasks besides \texttt{humanoidmaze-md} and \texttt{antmaze-giant}. For IQL, \citet{li_reinforcement_2025} use $\alpha=0.3$ for \texttt{cube-*}, $\alpha=10.0$ for \texttt{puzzle-3x3}, and $\tau=0.9$. For both \texttt{humanoidmaze-md} and \texttt{antmaze-giant}, we similarly use $\tau=0.9$ and $\alpha=10$. For ReBRAC, \citet{li_reinforcement_2025} use $\alpha=0.1$, and we use an actor behavioral cloning coefficient of $0.01$ and a critic behavioral cloning coefficient of $0.01$ for \texttt{humanoidmaze-md}, following \citet{park_flow_2025}, and an actor behavioral cloning coefficient of $0.003$ and a critic behavioral cloning coefficient of $0.01$ for \texttt{antmaze-giant}.

\subsection{Offline-to-online training and prior data sampling protocols.}
\label{app:offline_online_protocols}

\textbf{Offline-to-online training.} For each of the Best-of-$N$, FQL, and action chunking policies, we first sample from an offline dataset of demonstrations for $1\times10^6$ iterations for training. Afterwards, this offline dataset is used to fill a fixed-size replay buffer of capacity $2\times10^6$ transitions, which is then refreshed with a new transition collected from online experience every gradient step, for an additional $1\times10^6$ steps.

\textbf{Online training with prior data.} For gaussian policies, there is no offline learning phase; every gradient update is accompanied by an environment step. An empty replay buffer of capacity $2\times10^6$ transitions is filled only with online experience from the agent. Every gradient step, half the batch is sampled from an offline dataset of demonstrations, and the other half is sampled from the online replay buffer.

\textbf{Action chunking online.} Our action chunking training protocol takes inspiration from \citet{li_reinforcement_2025}. Every online training step, the chunked policy is sampled from, yielding a chunk of actions. Each action in the chunk is executed in the environment, one action per gradient update, until the chunk is exhausted.

\newpage

\section{Full Results}

\begin{table}[H]
  \vspace{-1\baselineskip}
  \caption{\textbf{OGBench results.} Each cell is the success rate (\%) at the end of online training (mean across seeds, equal-weight across the 5 tasks per group). \textbf{Bold} marks methods within 95\% of the row maximum; an overbar marks methods within 95\% of the per-actor-type maximum. The Total row is the equal-weight mean over the 5 OGBench groups. Cells without a bootstrap confidence interval are taken from prior work~\citep{li_reinforcement_2025}.}
  \centering
  \begin{small}
  \resizebox{\textwidth}{!}{%
  \setlength{\tabcolsep}{4pt}
  \renewcommand{\arraystretch}{1.7}
    \begin{tabular}{l|ccc|ccc|ccc|cc|cc}
      \toprule
       & \multicolumn{3}{c|}{\textbf{Best-of-$N$}} & \multicolumn{3}{c|}{\textbf{FQL}} & \multicolumn{3}{c|}{\textbf{Gaussian}} & \multicolumn{2}{c|}{\textbf{Action Chunking}} & \multicolumn{2}{c}{\textbf{Other baselines}} \\
        \cmidrule(lr){2-4} \cmidrule(lr){5-7} \cmidrule(lr){8-10} \cmidrule(lr){11-12} \cmidrule(lr){13-14}
       & TD & TD-$n$ & LQL & TD & TD-$n$ & LQL & TD & TD-$n$ & LQL & TD & LQL & ReBRAC & IQL \\
      \midrule
      \texttt{cube-double} & \makecell[tc]{68.3 \\[-3pt] {\scriptsize \textcolor{black!55}{[66,71]}}} & \makecell[tc]{67.9 \\[-3pt] {\scriptsize \textcolor{black!55}{[66,69]}}} & \makecell[tc]{\barnum{83.5} \\[-3pt] {\scriptsize \textcolor{black!55}{[81,86]}}} & \makecell[tc]{18.3 \\[-3pt] {\scriptsize \textcolor{black!55}{[16,21]}}} & \makecell[tc]{49.8 \\[-3pt] {\scriptsize \textcolor{black!55}{[46,53]}}} & \makecell[tc]{\barnum{86.2} \\[-3pt] {\scriptsize \textcolor{black!55}{[79,91]}}} & \makecell[tc]{0.2 \\[-3pt] {\scriptsize \textcolor{black!55}{[0,1]}}} & \makecell[tc]{0.0 \\[-3pt] {\scriptsize \textcolor{black!55}{[0,0]}}} & \makecell[tc]{\barnum{18.7} \\[-3pt] {\scriptsize \textcolor{black!55}{[18,19]}}} & \makecell[tc]{\barnum{\textbf{94.2}} \\[-3pt] {\scriptsize \textcolor{black!55}{[92,96]}}} & \makecell[tc]{\barnum{\textbf{95.1}} \\[-3pt] {\scriptsize \textcolor{black!55}{[94,97]}}} & 30 & 0 \\
      \texttt{cube-triple} & \makecell[tc]{11.1 \\[-3pt] {\scriptsize \textcolor{black!55}{[8,14]}}} & \makecell[tc]{0.7 \\[-3pt] {\scriptsize \textcolor{black!55}{[0,1]}}} & \makecell[tc]{\barnum{36.2} \\[-3pt] {\scriptsize \textcolor{black!55}{[33,39]}}} & \makecell[tc]{9.3 \\[-3pt] {\scriptsize \textcolor{black!55}{[7,11]}}} & \makecell[tc]{6.9 \\[-3pt] {\scriptsize \textcolor{black!55}{[2,12]}}} & \makecell[tc]{\barnum{31.2} \\[-3pt] {\scriptsize \textcolor{black!55}{[27,35]}}} & \makecell[tc]{0.0 \\[-3pt] {\scriptsize \textcolor{black!55}{[0,0]}}} & \makecell[tc]{0.0 \\[-3pt] {\scriptsize \textcolor{black!55}{[0,0]}}} & \makecell[tc]{\barnum{0.2} \\[-3pt] {\scriptsize \textcolor{black!55}{[0,0]}}} & \makecell[tc]{35.2 \\[-3pt] {\scriptsize \textcolor{black!55}{[33,37]}}} & \makecell[tc]{\barnum{\textbf{52.1}} \\[-3pt] {\scriptsize \textcolor{black!55}{[50,54]}}} & 0 & 0 \\
      \texttt{scene} & \makecell[tc]{90.6 \\[-3pt] {\scriptsize \textcolor{black!55}{[88,93]}}} & \makecell[tc]{\barnum{\textbf{95.1}} \\[-3pt] {\scriptsize \textcolor{black!55}{[93,97]}}} & \makecell[tc]{\barnum{\textbf{95.4}} \\[-3pt] {\scriptsize \textcolor{black!55}{[94,97]}}} & \makecell[tc]{\barnum{92.3} \\[-3pt] {\scriptsize \textcolor{black!55}{[83,98]}}} & \makecell[tc]{68.1 \\[-3pt] {\scriptsize \textcolor{black!55}{[61,75]}}} & \makecell[tc]{\barnum{93.9} \\[-3pt] {\scriptsize \textcolor{black!55}{[89,98]}}} & \makecell[tc]{69.4 \\[-3pt] {\scriptsize \textcolor{black!55}{[67,72]}}} & \makecell[tc]{\barnum{90.4} \\[-3pt] {\scriptsize \textcolor{black!55}{[86,94]}}} & \makecell[tc]{\barnum{89.5} \\[-3pt] {\scriptsize \textcolor{black!55}{[88,92]}}} & \makecell[tc]{\barnum{\textbf{97.3}} \\[-3pt] {\scriptsize \textcolor{black!55}{[96,98]}}} & \makecell[tc]{\barnum{\textbf{97.0}} \\[-3pt] {\scriptsize \textcolor{black!55}{[95,98]}}} & \textbf{99} & 39 \\
      \texttt{humanoid-md} & \makecell[tc]{72.8 \\[-3pt] {\scriptsize \textcolor{black!55}{[63,87]}}} & \makecell[tc]{71.1 \\[-3pt] {\scriptsize \textcolor{black!55}{[69,73]}}} & \makecell[tc]{\barnum{\textbf{96.2}} \\[-3pt] {\scriptsize \textcolor{black!55}{[94,98]}}} & \makecell[tc]{40.4 \\[-3pt] {\scriptsize \textcolor{black!55}{[39,42]}}} & \makecell[tc]{28.6 \\[-3pt] {\scriptsize \textcolor{black!55}{[25,32]}}} & \makecell[tc]{\barnum{65.2} \\[-3pt] {\scriptsize \textcolor{black!55}{[54,76]}}} & \makecell[tc]{20.5 \\[-3pt] {\scriptsize \textcolor{black!55}{[11,31]}}} & \makecell[tc]{7.9 \\[-3pt] {\scriptsize \textcolor{black!55}{[6,10]}}} & \makecell[tc]{\barnum{48.8} \\[-3pt] {\scriptsize \textcolor{black!55}{[36,62]}}} & \makecell[tc]{15.4 \\[-3pt] {\scriptsize \textcolor{black!55}{[8,25]}}} & \makecell[tc]{\barnum{22.2} \\[-3pt] {\scriptsize \textcolor{black!55}{[14,30]}}} & \makecell[tc]{3.5 \\[-3pt] {\scriptsize \textcolor{black!55}{[2,5]}}} & \makecell[tc]{23.1 \\[-3pt] {\scriptsize \textcolor{black!55}{[20,26]}}} \\
      \texttt{antmaze-giant} & \makecell[tc]{0.1 \\[-3pt] {\scriptsize \textcolor{black!55}{[0,0]}}} & \makecell[tc]{37.5 \\[-3pt] {\scriptsize \textcolor{black!55}{[34,41]}}} & \makecell[tc]{\barnum{44.8} \\[-3pt] {\scriptsize \textcolor{black!55}{[33,56]}}} & \makecell[tc]{4.8 \\[-3pt] {\scriptsize \textcolor{black!55}{[0,14]}}} & \makecell[tc]{28.8 \\[-3pt] {\scriptsize \textcolor{black!55}{[26,32]}}} & \makecell[tc]{\barnum{53.3} \\[-3pt] {\scriptsize \textcolor{black!55}{[51,56]}}} & \makecell[tc]{23.7 \\[-3pt] {\scriptsize \textcolor{black!55}{[12,36]}}} & \makecell[tc]{19.0 \\[-3pt] {\scriptsize \textcolor{black!55}{[18,20]}}} & \makecell[tc]{\barnum{\textbf{65.4}} \\[-3pt] {\scriptsize \textcolor{black!55}{[62,69]}}} & \makecell[tc]{22.2 \\[-3pt] {\scriptsize \textcolor{black!55}{[12,34]}}} & \makecell[tc]{\barnum{40.7} \\[-3pt] {\scriptsize \textcolor{black!55}{[37,45]}}} & \makecell[tc]{57.1 \\[-3pt] {\scriptsize \textcolor{black!55}{[53,62]}}} & \makecell[tc]{3.2 \\[-3pt] {\scriptsize \textcolor{black!55}{[2,4]}}} \\
      \midrule
      \textbf{Total} & \makecell[tc]{48.6 \\[-3pt] {\scriptsize \textcolor{black!55}{[46,51]}}} & \makecell[tc]{54.5 \\[-3pt] {\scriptsize \textcolor{black!55}{[54,55]}}} & \makecell[tc]{\barnum{\textbf{71.2}} \\[-3pt] {\scriptsize \textcolor{black!55}{[69,74]}}} & \makecell[tc]{33.0 \\[-3pt] {\scriptsize \textcolor{black!55}{[31,35]}}} & \makecell[tc]{36.4 \\[-3pt] {\scriptsize \textcolor{black!55}{[34,39]}}} & \makecell[tc]{\barnum{66.0} \\[-3pt] {\scriptsize \textcolor{black!55}{[63,69]}}} & \makecell[tc]{22.8 \\[-3pt] {\scriptsize \textcolor{black!55}{[20,26]}}} & \makecell[tc]{23.5 \\[-3pt] {\scriptsize \textcolor{black!55}{[22,24]}}} & \makecell[tc]{\barnum{44.5} \\[-3pt] {\scriptsize \textcolor{black!55}{[42,47]}}} & \makecell[tc]{52.9 \\[-3pt] {\scriptsize \textcolor{black!55}{[50,56]}}} & \makecell[tc]{\barnum{61.4} \\[-3pt] {\scriptsize \textcolor{black!55}{[60,63]}}} & 38 & 13 \\
      \bottomrule
    \end{tabular}%
  }
    \label{table:OGBench_success}
  \end{small}
\end{table}

\vspace{-1\baselineskip}
\begin{table}[H]
  
  \caption{\textbf{\texttt{humanoidmaze-giant} per-task success rate (\%)} at the end of online training (mean across seeds), for the runs in Figure~\ref{fig:humanoidmaze_giant}. \textbf{Bold} marks methods within 95\% of the per-row maximum. The Total row is the equal-weight mean over the 5 \texttt{humanoidmaze-giant} tasks. Numbers in brackets are the 95\% bootstrap CI; for the Total row, the bootstrap is stratified by task.}
  \centering
  \begin{small}
  \setlength{\tabcolsep}{4pt}
  \renewcommand{\arraystretch}{1.7}
    \begin{tabular}{lcccccc}
      \toprule
       & TD & TD-$4$ & TD-$8$ & TD-$16$ & TD-$64$ & LQL ($L=64$) \\
      \midrule
      \texttt{task1} & \makecell[tc]{0.0 \\[-3pt] {\scriptsize \textcolor{black!55}{[0,0]}}} & \makecell[tc]{18.0 \\[-3pt] {\scriptsize \textcolor{black!55}{[12,28]}}} & \makecell[tc]{1.3 \\[-3pt] {\scriptsize \textcolor{black!55}{[0,2]}}} & \makecell[tc]{0.0 \\[-3pt] {\scriptsize \textcolor{black!55}{[0,0]}}} & \makecell[tc]{0.0 \\[-3pt] {\scriptsize \textcolor{black!55}{[0,0]}}} & \makecell[tc]{\textbf{31.3} \\[-3pt] {\scriptsize \textcolor{black!55}{[2,58]}}} \\
      \texttt{task2} & \makecell[tc]{0.0 \\[-3pt] {\scriptsize \textcolor{black!55}{[0,0]}}} & \makecell[tc]{63.3 \\[-3pt] {\scriptsize \textcolor{black!55}{[50,76]}}} & \makecell[tc]{36.0 \\[-3pt] {\scriptsize \textcolor{black!55}{[24,56]}}} & \makecell[tc]{6.7 \\[-3pt] {\scriptsize \textcolor{black!55}{[4,10]}}} & \makecell[tc]{4.7 \\[-3pt] {\scriptsize \textcolor{black!55}{[2,6]}}} & \makecell[tc]{\textbf{97.3} \\[-3pt] {\scriptsize \textcolor{black!55}{[96,100]}}} \\
      \texttt{task3} & \makecell[tc]{0.0 \\[-3pt] {\scriptsize \textcolor{black!55}{[0,0]}}} & \makecell[tc]{7.3 \\[-3pt] {\scriptsize \textcolor{black!55}{[4,10]}}} & \makecell[tc]{4.7 \\[-3pt] {\scriptsize \textcolor{black!55}{[0,12]}}} & \makecell[tc]{0.0 \\[-3pt] {\scriptsize \textcolor{black!55}{[0,0]}}} & \makecell[tc]{0.0 \\[-3pt] {\scriptsize \textcolor{black!55}{[0,0]}}} & \makecell[tc]{\textbf{70.7} \\[-3pt] {\scriptsize \textcolor{black!55}{[66,76]}}} \\
      \texttt{task4} & \makecell[tc]{0.0 \\[-3pt] {\scriptsize \textcolor{black!55}{[0,0]}}} & \makecell[tc]{4.7 \\[-3pt] {\scriptsize \textcolor{black!55}{[4,6]}}} & \makecell[tc]{0.0 \\[-3pt] {\scriptsize \textcolor{black!55}{[0,0]}}} & \makecell[tc]{0.0 \\[-3pt] {\scriptsize \textcolor{black!55}{[0,0]}}} & \makecell[tc]{0.0 \\[-3pt] {\scriptsize \textcolor{black!55}{[0,0]}}} & \makecell[tc]{\textbf{80.7} \\[-3pt] {\scriptsize \textcolor{black!55}{[64,92]}}} \\
      \texttt{task5} & \makecell[tc]{0.0 \\[-3pt] {\scriptsize \textcolor{black!55}{[0,0]}}} & \makecell[tc]{\textbf{98.7} \\[-3pt] {\scriptsize \textcolor{black!55}{[96,100]}}} & \makecell[tc]{90.0 \\[-3pt] {\scriptsize \textcolor{black!55}{[82,96]}}} & \makecell[tc]{61.3 \\[-3pt] {\scriptsize \textcolor{black!55}{[54,66]}}} & \makecell[tc]{26.0 \\[-3pt] {\scriptsize \textcolor{black!55}{[20,36]}}} & \makecell[tc]{\textbf{98.7} \\[-3pt] {\scriptsize \textcolor{black!55}{[96,100]}}} \\
      \midrule
      \textbf{Total} & \makecell[tc]{0.0 \\[-3pt] {\scriptsize \textcolor{black!55}{[0,0]}}} & \makecell[tc]{38.4 \\[-3pt] {\scriptsize \textcolor{black!55}{[35,41]}}} & \makecell[tc]{26.4 \\[-3pt] {\scriptsize \textcolor{black!55}{[23,30]}}} & \makecell[tc]{13.6 \\[-3pt] {\scriptsize \textcolor{black!55}{[12,15]}}} & \makecell[tc]{6.1 \\[-3pt] {\scriptsize \textcolor{black!55}{[5,8]}}} & \makecell[tc]{\textbf{75.7} \\[-3pt] {\scriptsize \textcolor{black!55}{[70,81]}}} \\
      \bottomrule
    \end{tabular}
    \label{table:batch_scale_giant}
  \end{small}
\end{table}

\vspace{-1\baselineskip}
\begin{table}[H]
  \caption{\textbf{RoboMimic results.} Each cell is the success rate (\%) at the end of online training (mean across seeds). The Total row is the equal-weight mean of Square and Can. \textbf{Bold} marks methods within 95\% of the row maximum; an overbar marks methods within 95\% of the per-actor maximum.}
  \centering
  \begin{small}
  \setlength{\tabcolsep}{4pt}
  \renewcommand{\arraystretch}{1.7}
    \begin{tabular}{l|ccc|ccc|ccc|cc}
      \toprule
       & \multicolumn{3}{c|}{\textbf{Best-of-$N$}} & \multicolumn{3}{c|}{\textbf{FQL}} & \multicolumn{3}{c|}{\textbf{Gaussian}} & \multicolumn{2}{c}{\textbf{Action Chunking}} \\
        \cmidrule(lr){2-4} \cmidrule(lr){5-7} \cmidrule(lr){8-10} \cmidrule(lr){11-12}
       & TD & TD-$n$ & LQL & TD & TD-$n$ & LQL & TD & TD-$n$ & LQL & TD & LQL \\
      \midrule
      Square & \makecell[tc]{23.0 \\[-3pt] {\scriptsize \textcolor{black!55}{[19,27]}}} & \makecell[tc]{23.0 \\[-3pt] {\scriptsize \textcolor{black!55}{[16,28]}}} & \makecell[tc]{\barnum{42.5} \\[-3pt] {\scriptsize \textcolor{black!55}{[39,46]}}} & \makecell[tc]{0.0 \\[-3pt] {\scriptsize \textcolor{black!55}{[0,0]}}} & \makecell[tc]{15.5 \\[-3pt] {\scriptsize \textcolor{black!55}{[8,22]}}} & \makecell[tc]{\barnum{61.0} \\[-3pt] {\scriptsize \textcolor{black!55}{[60,63]}}} & \makecell[tc]{7.5 \\[-3pt] {\scriptsize \textcolor{black!55}{[3,12]}}} & \makecell[tc]{\barnum{\textbf{92.5}} \\[-3pt] {\scriptsize \textcolor{black!55}{[90,95]}}} & \makecell[tc]{\barnum{\textbf{92.5}} \\[-3pt] {\scriptsize \textcolor{black!55}{[90,94]}}} & \makecell[tc]{20.5 \\[-3pt] {\scriptsize \textcolor{black!55}{[19,22]}}} & \makecell[tc]{\barnum{23.5} \\[-3pt] {\scriptsize \textcolor{black!55}{[16,31]}}} \\
      Can & \makecell[tc]{81.0 \\[-3pt] {\scriptsize \textcolor{black!55}{[77,84]}}} & \makecell[tc]{80.5 \\[-3pt] {\scriptsize \textcolor{black!55}{[77,86]}}} & \makecell[tc]{\barnum{\textbf{87.0}} \\[-3pt] {\scriptsize \textcolor{black!55}{[86,88]}}} & \makecell[tc]{3.0 \\[-3pt] {\scriptsize \textcolor{black!55}{[0,6]}}} & \makecell[tc]{69.5 \\[-3pt] {\scriptsize \textcolor{black!55}{[63,73]}}} & \makecell[tc]{\barnum{\textbf{90.0}} \\[-3pt] {\scriptsize \textcolor{black!55}{[84,96]}}} & \makecell[tc]{2.5 \\[-3pt] {\scriptsize \textcolor{black!55}{[1,4]}}} & \makecell[tc]{\barnum{82.5} \\[-3pt] {\scriptsize \textcolor{black!55}{[80,88]}}} & \makecell[tc]{\barnum{84.0} \\[-3pt] {\scriptsize \textcolor{black!55}{[78,90]}}} & \makecell[tc]{\barnum{82.0} \\[-3pt] {\scriptsize \textcolor{black!55}{[80,84]}}} & \makecell[tc]{\barnum{\textbf{85.5}} \\[-3pt] {\scriptsize \textcolor{black!55}{[80,90]}}} \\
      \midrule
      \textbf{Total} & \makecell[tc]{52.0 \\[-3pt] {\scriptsize \textcolor{black!55}{[50,55]}}} & \makecell[tc]{51.7 \\[-3pt] {\scriptsize \textcolor{black!55}{[48,55]}}} & \makecell[tc]{\barnum{64.8} \\[-3pt] {\scriptsize \textcolor{black!55}{[63,66]}}} & \makecell[tc]{1.5 \\[-3pt] {\scriptsize \textcolor{black!55}{[0,3]}}} & \makecell[tc]{42.5 \\[-3pt] {\scriptsize \textcolor{black!55}{[38,47]}}} & \makecell[tc]{\barnum{75.5} \\[-3pt] {\scriptsize \textcolor{black!55}{[72,78]}}} & \makecell[tc]{5.0 \\[-3pt] {\scriptsize \textcolor{black!55}{[3,8]}}} & \makecell[tc]{\barnum{\textbf{87.5}} \\[-3pt] {\scriptsize \textcolor{black!55}{[86,90]}}} & \makecell[tc]{\barnum{\textbf{88.2}} \\[-3pt] {\scriptsize \textcolor{black!55}{[85,91]}}} & \makecell[tc]{51.2 \\[-3pt] {\scriptsize \textcolor{black!55}{[50,52]}}} & \makecell[tc]{\barnum{54.5} \\[-3pt] {\scriptsize \textcolor{black!55}{[50,59]}}} \\
      \bottomrule
    \end{tabular}
    \label{table:robomimic_success}
  \end{small}
\end{table}

\newpage

\subsection{Comparison to Optimality Tightening}
\label{app:ot_cmp}

We cloned the optimality tightening codebase of \citet{he_learning_2016} and updated their discrete-action-only implementation to use a continuous actor with a tanh-squashed Gaussian (like the Gaussian actor we use for LQL and TD($-n$)); our adapted code is available at \url{https://github.com/armaan-abraham/Q-Optimality-Tightening}. We used default hyperparameters from the repository, after matching the network size of the critic and the actor to those used in our experiments. We evaluated this agent for the same number of online training steps ($2 \times 10^6$) as the gaussian LQL agent, and compare them in Figure~\ref{fig:ot_cmp_success}.

\begin{figure}[H]
    \centering
    \includegraphics[width=\linewidth]{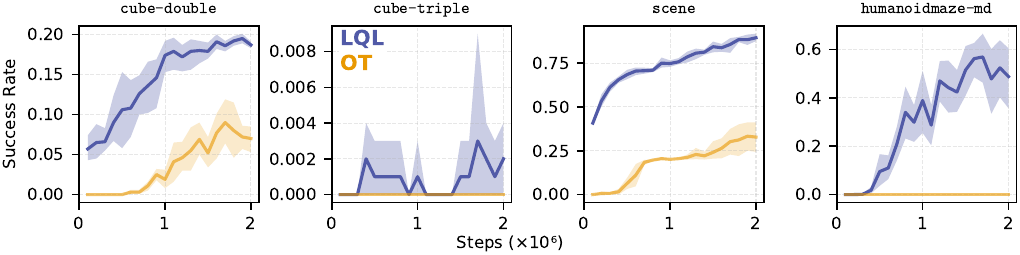}
    \caption{Mean success rate of LQL and OT across 4 task groups, evaluated on all 5 tasks in the group and with 4 seeds per task.}
    \label{fig:ot_cmp_success}
\end{figure}

\begin{table}[H]
  \caption{\textbf{LQL vs.\ OT.} Each cell is the success rate (\%) at the end of online training (mean across seeds, equal-weight across the 5 tasks per group). \textbf{Bold} marks the per-row maximum. The Total row is the equal-weight mean over the 4 task groups. Numbers in brackets are the 95\% bootstrap CI; for the Total row, the bootstrap is stratified by task. Results for LQL are identical to the Gaussian column in Table~\ref{table:OGBench_success}.}
  \centering
  \begin{small}
  \setlength{\tabcolsep}{4pt}
  \renewcommand{\arraystretch}{1.7}
    \begin{tabular}{lcc}
      \toprule
       & LQL & OT \\
      \midrule
      \texttt{cube-double} & \makecell[tc]{\textbf{18.7} \\[-3pt] {\scriptsize \textcolor{black!55}{[18,19]}}} & \makecell[tc]{7.0 \\[-3pt] {\scriptsize \textcolor{black!55}{[5,9]}}} \\
      \texttt{cube-triple} & \makecell[tc]{\textbf{0.2} \\[-3pt] {\scriptsize \textcolor{black!55}{[0,0]}}} & \makecell[tc]{0.0 \\[-3pt] {\scriptsize \textcolor{black!55}{[0,0]}}} \\
      \texttt{scene} & \makecell[tc]{\textbf{89.5} \\[-3pt] {\scriptsize \textcolor{black!55}{[88,92]}}} & \makecell[tc]{32.8 \\[-3pt] {\scriptsize \textcolor{black!55}{[25,41]}}} \\
      \texttt{humanoid-md} & \makecell[tc]{\textbf{48.8} \\[-3pt] {\scriptsize \textcolor{black!55}{[36,62]}}} & \makecell[tc]{0.0 \\[-3pt] {\scriptsize \textcolor{black!55}{[0,0]}}} \\
      \midrule
      \textbf{Total} & \makecell[tc]{\textbf{39.3} \\[-3pt] {\scriptsize \textcolor{black!55}{[36,43]}}} & \makecell[tc]{10.0 \\[-3pt] {\scriptsize \textcolor{black!55}{[8,12]}}} \\
      \bottomrule
    \end{tabular}
    \label{table:ot_lql_cmp}
  \end{small}
\end{table}

\subsection{Stochastic environments}

\label{app:stochastic_envs_details}

For the experiments in Section~\ref{subsec:stochastic_envs} (Figure~\ref{fig:stochastic_envs}), we created stochastic versions of OGBench environments by adding independent Gaussian noise to each dimension of the action before executing it in the environment, clipping to the action range $[-1, 1]$. For each degree of stochasticity, labeled by the standard deviation of the corresponding Gaussian, we recollected a separate offline dataset using the method from \citet{park_ogbench_2025} with this updated stochasticity, yielding a dataset of the same size as in Table~\ref{table:task_characteristics}.

\begin{figure}[H]
    \centering
    \includegraphics[width=0.75\linewidth]{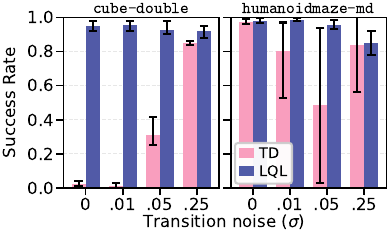}
    \caption{\textbf{LQL continues to perform on par with or better than TD in stochastic environments.} Each panel shows \texttt{task2} of the listed task group.}
    \label{fig:stochastic_envs}
\end{figure}

\subsection{Hinge coefficient sweep}

\begin{figure}[H]
    \centering
    \includegraphics[width=0.75\linewidth]{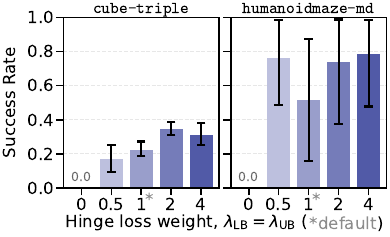}
    \caption{\textbf{Hinge coefficient sweep.} FQL actor was used with otherwise the same hyperparameters as Tables \ref{table:hyperparams}, \ref{table:alpha_sets}. Each plot shows \texttt{task3} of the listed task group. Four seeds.}
    \label{fig:lambda_sweep}
\end{figure}

\newpage

\subsection{Isolating the effects of trajectory sampling and hinge penalties}
\label{app:sample_ctrl}

LQL differs from standard TD both in \emph{what} is sampled (short trajectories) and \emph{what loss} is applied (additional hinge penalties). To disentangle these, we set the hinge-loss coefficients in Eq.~\ref{eq:lql_loss_sample_revised} to zero, which reduces LQL to TD learning on sampled trajectories rather than individually sampled transitions. Figure~\ref{fig:sample_ctrl} shows that LQL's hinge penalties yield further gains beyond this trajectory sampling control, supporting the role of the long-horizon backstop. Surprisingly, the trajectory sampling control in some cases performs better than standard transition-level TD. This may reflect the mechanism observed in \citet{fu_compute-optimal_2025}, where smaller batch sizes can yield better performance for TD due to reduced overfitting to inaccurate target network predictions. In this case, sampling transitions within trajectories may reduce the effective diversity of the batch in a way that produces the same effect.

\begin{figure}[H]
    \centering
    \includegraphics[width=\linewidth]{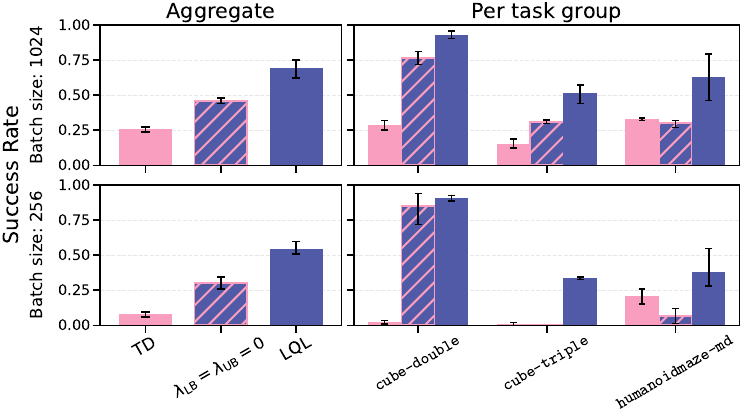}
    \caption{
        \textbf{LQL's gains are not explained by trajectory sampling alone; the hinge backstop contributes beyond this control.} Success rates for FQL policies are averaged over tasks 1--3 in \texttt{cube-double}, \texttt{cube-triple}, and \texttt{humanoidmaze-md}. The top row uses the configuration with the hyperparameters used in the rest of the paper, including a batch size of 1024 (LQL: 128 trajectories of length 8). The bottom row uses batch size 256 (LQL: 64 trajectories of length 4), with $\alpha=50$ for \texttt{cube-double} and \texttt{cube-triple}, and $\alpha=10$ for \texttt{humanoidmaze-md}.
    }
    \label{fig:sample_ctrl}
\end{figure}

\subsection{Batch-size-controlled trajectory length sweep}
\label{app:horizon_ctrl}

Under a fixed compute-per-batch constraint, increasing the trajectory length requires reducing the number of sampled trajectories per batch. Figure~\ref{fig:horizon_ctrl} shows that when the batch size is held constant, the optimal trajectory length varies by task, with performance degrading at longer trajectory lengths in some environments. Taken in isolation, this degradation could be attributed to two potential causes: (a) instability from longer-horizon hinge penalties, or (b) the reduction in the number of sampled trajectories per batch, and the resulting reduced diversity of behavioral policies represented in the batch, yielding noisier hinge, TD, and actor losses alike. However, the result in Figure~\ref{fig:horizon_scale}, in which the trajectory length is scaled without a corresponding reduction in the number of sampled trajectories per batch, shows consistent performance improvements with longer trajectories, supporting (b) over (a).

\begin{figure}[H]
    \centering
    \includegraphics[width=0.75\linewidth]{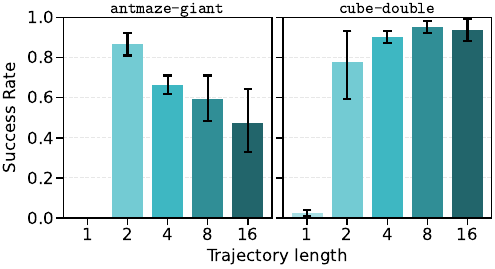}
    \caption{
        \textbf{With fixed batch size, optimal LQL trajectory length varies by task.} Each plot shows performance of FQL policies on \texttt{task2} of the listed environment. We keep the batch size fixed at 1024.
    }
    \label{fig:horizon_ctrl}
\end{figure}

\subsection{Scaling trajectory length for LQL vs. batch size for TD}
\label{app:horizon_batch_scale}

We conduct more experiments of the form shown in Figure~\ref{fig:horizon_scale}, namely scaling trajectory length, additionally showing the effect of an equivalent increase in batch size for TD. These experiments are conducted at a smaller batch size for both methods, with FQL actors with $\alpha=10$ for \texttt{humanoidmaze-md} and $\alpha=50$ for \texttt{cube-double}. Reflecting previous work~\citep{fu_compute-optimal_2025}, TD does not respond favorably to batch size scaling, while LQL consistently improves with longer trajectories (Figure~\ref{fig:horizon_batch_scale}).

\begin{figure}[H]
    \centering
    \includegraphics[width=\linewidth]{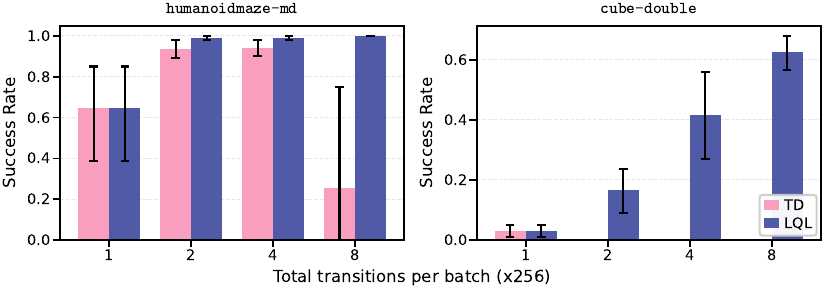}
    \caption{
        \textbf{Scaling compute via longer LQL trajectories is more effective than scaling TD with more independent transitions.} For matched scaling factors on the x-axis (which also correspond to LQL trajectory length), LQL benefits consistently from longer segments, while TD does not reliably improve with larger batches of individually sampled transitions. Each panel shows \texttt{task2} of the listed environment. Identical results are used for TD with batch size 256 and LQL with horizon length 1, because these two methods are identical in this case.
    }
    \label{fig:horizon_batch_scale}
\end{figure}

\subsection{Hinge penalty activation vs.\ hinge distance}
\label{app:hinge_vs_t}

While training the LQL-with-FQL-actor agent, we recorded the hinge penalty activation frequency and magnitude throughout training, grouped by whether the penalty was a lower bound or upper bound and by the distance over which the penalty was computed in Figure~\ref{fig:hinge_vs_t}. The penalty magnitude is measured as the average over all pairwise comparisons, including those that are zero (i.e., unactivated). The y-axis range of 2 to 8 for the lower bound and 0 to 6 for the upper bound reflects that the lower bound skips the hinge comparison for the next state since it is already included in the TD loss, while the upper bound includes the zero comparison because the same-state upper bound uses the target network evaluated on the policy-generated action at the same state, which is already computed for the TD update at the previous transition. We see that broadly the hinge penalties activate more frequently for shorter-distance comparisons, but conversely tend to be larger in magnitude for longer-distance comparisons. One vaguely apparent pattern is that the hinge penalties amplify after a small delay from the beginning of online training. 

\begin{figure}[H]
    \centering
    \includegraphics[width=\linewidth]{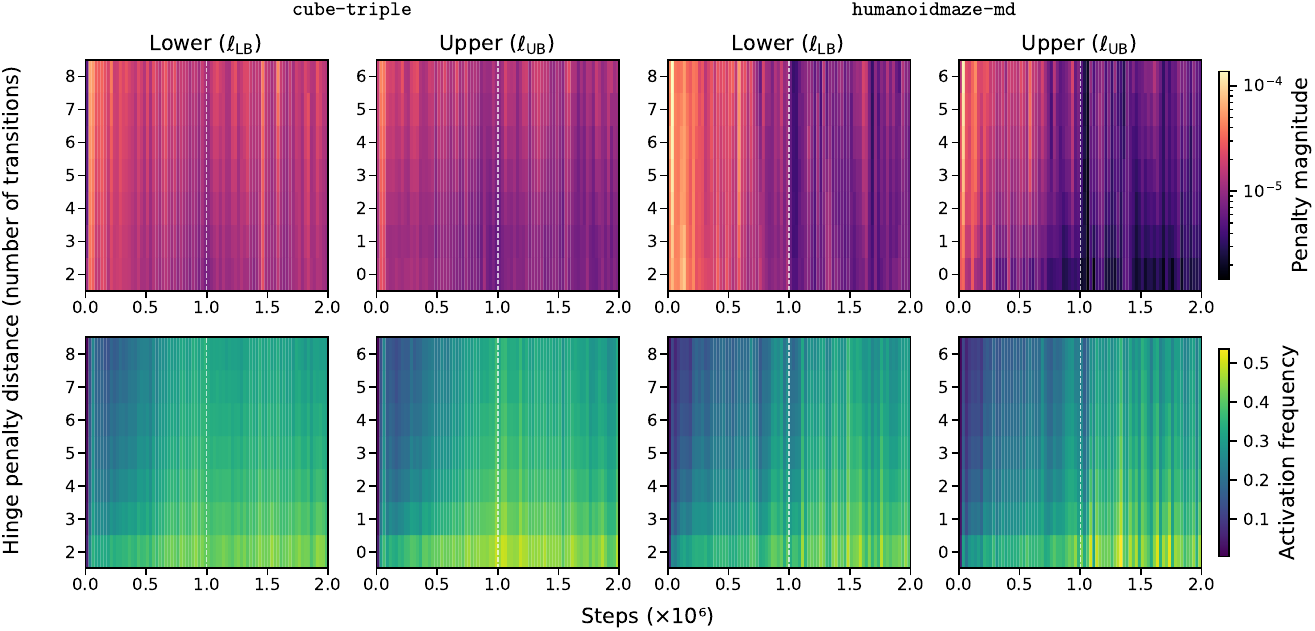}
    \caption{\textbf{Hinge penalty activation frequency and magnitude show task and training stage-dependent patterns}. Penalty magnitude is normalized by $Q_\theta^2$ (using the batch-mean $Q_\theta$ at each step). Offline-online training transition marked by white dashed line. Averaged over two seeds per task (\texttt{task1} of each task group).}
    \label{fig:hinge_vs_t}
\end{figure}

\subsection{Computational requirements}
\label{app:compute_requirements}

Across the policy extraction families used in the main experiments, LQL incurs a $4.7\%$ per-update slowdown on average over TD (Table~\ref{table:runtimes}).

\begin{table}[H]
  \caption{\textbf{Runtime comparison.} Iterations per second on a NVIDIA A40 GPU. Offline iter/sec for Best-of-$N$ and FQL; online for Gaussian.}
  \centering
  \begin{small}
    \begin{tabular}{lccc}
      \toprule
      Policy type & TD (iter/sec) & LQL (iter/sec) & Slowdown  \\
      \midrule
      Best-of-$N$        & 57.1 & 56.0 & 2.0\% \\
      FQL                & 288.5 & 273.1 & 5.3\% \\
      Gaussian           & 138.8 & 132.5 & 4.5\% \\
      \midrule
      Average   & 161.5 & 153.9 & 4.7\% \\
      \bottomrule
    \end{tabular}
    \label{table:runtimes}
  \end{small}
\end{table}

For the \texttt{humanoidmaze-giant} runs (Figure~\ref{fig:humanoidmaze_giant}) which use a longer trajectory length $L=64$, one potential concern is that the $\mathcal{O}(L^2)$ pairwise comparisons would dramatically slow down training. Table~\ref{table:runtimes_high_L} shows that at this larger trajectory length, the additional runtime cost of LQL remains negligible.

\begin{table}[H]
  \caption{\textbf{Offline iterations per second} on a NVIDIA H100 GPU, for \texttt{humanoidmaze-giant} experiments in Figure~\ref{fig:humanoidmaze_giant}.}
  \centering
  \begin{small}
    \begin{tabular}{lc}
      \toprule
      Method & iter/sec \\
      \midrule
      TD            & 29.91 \\
      TD-$4$        & 29.36 \\
      TD-$8$        & 30.06 \\
      TD-$16$       & 30.04 \\
      TD-$64$       & 29.76 \\
      LQL ($L=64$)  & 29.72 \\
      \bottomrule
    \end{tabular}
    \label{table:runtimes_high_L}
  \end{small}
\end{table}

\newpage

\section{Theoretical Analysis}
\label{app:theory}

\subsection{Bounded false penalties due to stochasticity}
\label{app:false_penalties}

We analyze both hinge penalties (LB, then UB) at $Q_\theta=Q^*$ in stochastic environments, and denote these as false penalties. We establish bounds on false penalties that are independent of the number of steps $L$ over which the penalties are computed.

\paragraph{LB hinge.}
The $L$-step LB violation signal at trajectory position $i$ is
\[
Z_L = G_{i:i+L} + \gamma^L Q^*(s_{i+L}, a^*(s_{i+L})) - Q^*(s_i, a_i).
\]
Telescoping via the Bellman equation gives
\[
Z_L = \sum_{k=0}^{L-1} \gamma^k \epsilon_{i+k} \;-\; \gamma\sum_{k=0}^{L-2}\gamma^k \Delta_{i+k+1},
\]
where $\epsilon_j = r_j + \gamma Q^*(s_{j+1}, a^*(s_{j+1})) - Q^*(s_j, a_j)$ is the 1-step Bellman noise (with $\mathbb{E}[\epsilon_j \mid s_j, a_j] = 0$) and $\Delta_j = Q^*(s_j, a^*(s_j)) - Q^*(s_j, a_j) \geq 0$ is the suboptimality gap.

\paragraph{Per-step decomposition.}
Re-indexing the gap sum (substituting $j = k+1$) yields
\[
Z_L = \epsilon_i + \sum_{j=1}^{L-1} \gamma^j\bigl(\epsilon_{i+j} - \Delta_{i+j}\bigr).
\]
Define per-step terms
\[
W_j \;\triangleq\;
\begin{cases}
\epsilon_{i} & j = 0,\\
\epsilon_{i+j} - \Delta_{i+j} & j = 1, \ldots, L-1,
\end{cases}
\]
so that $Z_L = \sum_{j=0}^{L-1} \gamma^j W_j$.

\paragraph{Assumption.}
\emph{Bounded rewards}: $|r_j| \leq R_{\max}$ for all $j$.

This implies $|Q^*(s,a)| \leq Q_{\max} \triangleq R_{\max}/(1-\gamma)$ and therefore $|\epsilon_j| \leq R_{\max} + (1+\gamma)Q_{\max}$ and $|\Delta_j| \leq 2Q_{\max}$.  Consequently, every $W_j$ is bounded:
\[
|W_j| \;\leq\; M \;\triangleq\; R_{\max} + (1+\gamma)Q_{\max} + 2Q_{\max},
\]
and in particular $\mathbb{E}[W_j^2] \leq M^2$ for all $j$.

\paragraph{Mean of $Z_L$.}
All expectations below are over the replay trajectory distribution.
Since $\mathbb{E}[\epsilon_{i+j} \mid s_{i+j}, a_{i+j}] = 0$ and $\Delta_{i+j} \geq 0$, for $L \geq 2$:
\[
\mathbb{E}[Z_L]
= \sum_{j=0}^{L-1} \gamma^j \mathbb{E}[W_j]
= -\sum_{j=1}^{L-1}\gamma^j \,\mathbb{E}[\Delta_{i+j}]
\;\triangleq\; -\mu_D(L) \;\leq\; 0.
\]
Defining $\bar\Delta_{\mathrm{LB}} \triangleq \mathbb{E}[\Delta_{i+1}]$, $\mu_D(L) \geq \gamma\bar\Delta_{\mathrm{LB}}$ because every $\Delta_{i+j} \geq 0$.

\paragraph{Variance bound.}
Expanding $\operatorname{Var}(Z_L)$:
\[
\operatorname{Var}(Z_L)
= \sum_{j=0}^{L-1}\sum_{j'=0}^{L-1} \gamma^{j+j'}\operatorname{Cov}(W_j, W_{j'}).
\]
By Cauchy--Schwarz,
\[
|\operatorname{Cov}(W_j, W_{j'})| \;\leq\; \sqrt{\operatorname{Var}(W_j)\,\operatorname{Var}(W_{j'})} \;\leq\; \sqrt{\mathbb{E}[W_j^2]\,\mathbb{E}[W_{j'}^2]} \;\leq\; M^2.
\]
Therefore
\[
\operatorname{Var}(Z_L)
\;\leq\; M^2 \sum_{j=0}^{L-1}\sum_{j'=0}^{L-1} \gamma^{j+j'}
= M^2 \!\left(\sum_{j=0}^{L-1}\gamma^j\right)^{\!2}
\;\leq\; M^2 \!\left(\frac{1}{1-\gamma}\right)^{\!2}
\;\triangleq\; V^\infty,
\]
which is finite and independent of $L$.

\begin{lemma}
\label{lemma:cantelli}
For any random variable $X$ with $\mathbb{E}[X] = -\mu$ $(\mu \geq 0)$ and $\operatorname{Var}(X) \leq V$:
\begin{align}
\Pr(X > 0) &\leq \frac{V}{V + \mu^2},\label{eq:cantelli}\\[4pt]
\mathbb{E}\bigl[[X]_+\bigr] &\leq \frac{V}{2\bigl(\!\sqrt{V + \mu^2} + \mu\bigr)}.\label{eq:cantelli_exp}
\end{align}
\end{lemma}

\begin{proof}
Inequality~\eqref{eq:cantelli} is the Cantelli inequality \citep{cantelli1928} applied to $\tilde X = X + \mu$ (mean-zero, variance $\leq V$) at threshold $\mu$.

For~\eqref{eq:cantelli_exp}, use the identity $[X]_+ = (X + |X|)/2$, so
\[
\mathbb{E}\bigl[[X]_+\bigr] = \frac{\mathbb{E}[X] + \mathbb{E}\bigl[|X|\bigr]}{2} = \frac{-\mu + \mathbb{E}\bigl[|X|\bigr]}{2}.
\]
To bound $\mathbb{E}\bigl[|X|\bigr]$, apply Cauchy--Schwarz:
\[
\mathbb{E}\bigl[|X|\bigr] = \mathbb{E}\bigl[|X| \cdot 1\bigr] \leq \sqrt{\mathbb{E}\bigl[|X|^2\bigr]}\;\sqrt{\mathbb{E}\bigl[1^2\bigr]} = \sqrt{\mathbb{E}\bigl[X^2\bigr]} = \sqrt{V + \mu^2}.
\]
Substituting:
\[
\mathbb{E}\bigl[[X]_+\bigr] \leq \frac{\sqrt{V + \mu^2} - \mu}{2} \cdot \frac{\sqrt{V + \mu^2} + \mu}{\sqrt{V + \mu^2} + \mu} = \frac{V}{2\bigl(\!\sqrt{V + \mu^2} + \mu\bigr)}.
\]
\end{proof}

\paragraph{False penalty bound.}
The LB hinge penalty produces a false violation whenever $Z_L > 0$. Applying Lemma~\ref{lemma:cantelli} to $Z_L$ for $L \geq 2$, $\mu = \mu_D(L)$, $V \leq V^\infty$, and using $\mu_D(L) \geq \gamma\bar\Delta_{\mathrm{LB}}$:
\begin{align}
\Pr([Z_L]_+^2>0) = \Pr(Z_L > 0) \leq \frac{V^\infty}{V^\infty + \gamma^2\bar\Delta_{\mathrm{LB}}^2},\\
\mathbb{E}\bigl[[Z_L]_+\bigr] \leq \frac{V^\infty}{2\bigl(\!\sqrt{V^\infty + \gamma^2\bar\Delta_{\mathrm{LB}}^2} + \gamma\bar\Delta_{\mathrm{LB}}\bigr)}.
\end{align}
For the expected squared penalty, since $|W_j| \leq M$,
\[
[Z_L]_+ \;\leq\; |Z_L| \;\leq\; M\sum_{j=0}^{L-1}\gamma^j \;\leq\; \frac{M}{1-\gamma} \;=\; \sqrt{V^\infty},
\]
and therefore $[Z_L]_+^2 \leq \sqrt{V^\infty}\,[Z_L]_+$. Taking expectations:
\[
\mathbb{E}\bigl[[Z_L]_+^2\bigr] \;\leq\; \sqrt{V^\infty}\,\mathbb{E}\bigl[[Z_L]_+\bigr] \;\leq\; \frac{(V^\infty)^{3/2}}{2\bigl(\!\sqrt{V^\infty + \gamma^2\bar\Delta_{\mathrm{LB}}^2} + \gamma\bar\Delta_{\mathrm{LB}}\bigr)}.
\]
All bounds are finite, independent of $L$, and decrease with the suboptimality of the experience in the replay buffer $\bar\Delta_{\mathrm{LB}}$.

\paragraph{UB hinge.}
Define the $L$-step UB violation signal at trajectory position $i$ as
\[
U_L = G_{i:i+L} + \gamma^L Q^*(s_{i+L}, a_{i+L}) - Q^*(s_i, a^*(s_i)).
\]
Using the same telescoping as for $Z_L$ and substituting $Q^*(s_i,a^*(s_i)) = Q^*(s_i,a_i) + \Delta_i$, we obtain
\[
U_L
= \sum_{k=0}^{L-1}\gamma^k \epsilon_{i+k} \;-\; \sum_{k=0}^{L}\gamma^k \Delta_{i+k}.
\]

\paragraph{Per-step decomposition.}
Define
\[
\widetilde W_k \;\triangleq\;
\begin{cases}
\epsilon_{i+k} - \Delta_{i+k} & k = 0, \ldots, L-1,\\
-\Delta_{i+L} & k = L,
\end{cases}
\]
so that $U_L = \sum_{k=0}^{L} \gamma^k \widetilde W_k$.  Since $|\epsilon_{i+k}| \leq R_{\max} + (1+\gamma)Q_{\max}$ and $|\Delta_{i+k}| \leq 2Q_{\max}$, every $\widetilde W_k$ satisfies $|\widetilde W_k| \leq M$ with the same constant $M = R_{\max} + (1+\gamma)Q_{\max} + 2Q_{\max}$ as in the lower-bound case.

\paragraph{Mean of $U_L$.}
Since $\mathbb{E}[\epsilon_{i+k} \mid s_{i+k}, a_{i+k}] = 0$ and $\Delta_{i+k} \geq 0$,
\[
\mathbb{E}[U_L]
= -\sum_{k=0}^{L}\gamma^k \mathbb{E}[\Delta_{i+k}]
\;\triangleq\; -\mu_U(L)
\;\leq\; 0.
\]
Defining $\bar\Delta_{\mathrm{UB}} \triangleq \mathbb{E}[\Delta_i]$, we have $\mu_U(L) \geq \bar\Delta_{\mathrm{UB}}$ since every summand is nonnegative.

\paragraph{Variance bound.}
By the same Cauchy--Schwarz argument as for $Z_L$, using $|\widetilde W_k| \leq M$:
\[
\operatorname{Var}(U_L)
\;\leq\; M^2 \!\left(\sum_{k=0}^{L}\gamma^k\right)^{\!2}
\;\leq\; M^2 \!\left(\frac{1}{1-\gamma}\right)^{\!2}
\;=\; V^\infty.
\]

\paragraph{False penalty bound.}
Applying Lemma~\ref{lemma:cantelli} to $U_L$ with $\mu = \mu_U(L) \geq \bar\Delta_{\mathrm{UB}}$ and $V \leq V^\infty$:
\begin{align}
\Pr([U_L]_+^2>0)  &\leq \frac{V^\infty}{V^\infty + \bar\Delta_{\mathrm{UB}}^2},\\
\mathbb{E}\bigl[[U_L]_+^2\bigr] &\leq \frac{(V^\infty)^{3/2}}{2\bigl(\!\sqrt{V^\infty + \bar\Delta_{\mathrm{UB}}^2} + \bar\Delta_{\mathrm{UB}}\bigr)}.
\end{align}
Both bounds are finite, independent of $L$, and decrease with the suboptimality gap $\bar\Delta_{\mathrm{UB}}$ in the replay buffer. The bounds, at the very least, show that sending trajectory length $\to \infty$ will not also send the false penalty $\to \infty$.

\end{document}